\documentclass[sigconf]{acmart}
\makeatletter
\@ifundefined{Bbbk}{}{}
\makeatother

\usepackage[utf8]{inputenc} 
\usepackage[T1]{fontenc}    
\usepackage{hyperref}       
\usepackage{url}            
\usepackage{booktabs}       
\usepackage{amsfonts}       
\usepackage{amsthm}         

\usepackage{nicefrac}       
\usepackage{microtype}      
\usepackage{xcolor}         
\usepackage{threeparttablex}

\usepackage{xspace}  
\usepackage{amsmath} 
\usepackage{multirow}
\usepackage{array}
\usepackage[ruled,vlined,linesnumbered]{algorithm2e}
\usepackage{adjustbox}
\usepackage{graphicx} 
\usepackage{amssymb}
\usepackage{enumitem}
\usepackage{wrapfig}
\usepackage{mathtools} 
\usepackage[symbol]{footmisc}
\usepackage{titletoc}
\usepackage{balance}

\newcommand{\V}{\mathcal{V}}
\newcommand{\X}{\mathcal{X}}
\newcommand{\Y}{\mathcal{Y}}

\newcommand{\E}{\mathcal{E}}

\newcommand{\R}{\mathbb{R}}

\newcommand{\G}{\mathcal{G}}
\newcommand{\T}{\mathcal{T}}
\newcommand{\W}{\mathbf{W}}
\newcommand{\LLM}{ \mathcal{L}_{\tau}}
\newcommand{\method}{\textsc{SaVe-TAG}}

\theoremstyle{plain}
\newtheorem{theorem}{Theorem}[section]
\newtheorem{definition}[theorem]{Definition}

\definecolor{hlcolor}{RGB}{255,240,160}

\title{
\method: 
LLM-based Interpolation\\ for Long-Tailed Text-Attributed Graphs
}

\author{Leyao Wang}
\affiliation{%
  \institution{Yale University}
  \city{New Haven}
  \state{CT}
  \country{USA}
}
\email{leyao.wang.lw855@yale.edu}

\author{Yu Wang}
\affiliation{%
  \institution{University of Oregon}
  \city{Eugene}
  \state{OR}
  \country{USA}
}
\email{yuwang@uoregon.edu}

\author{Bo Ni}
\affiliation{%
  \institution{Vanderbilt University}
  \city{Nashville}
  \state{TN}
  \country{USA}
}
\email{bo.ni@vanderbilt.edu}

\author{Yuying Zhao}
\affiliation{%
  \institution{Vanderbilt University}
  \city{Nashville}
  \state{TN}
  \country{USA}
}
\email{yuying.zhao@vanderbilt.edu}

\author{Hanyu Wang}
\affiliation{%
  \institution{Renmin University of China}
  \city{Beijing}
  \country{China}
}
\email{hy.wang@ruc.edu.cn}

\author{Yao Ma}
\affiliation{%
  \institution{Rensselaer Polytechnic Institute}
  \city{Troy}
  \state{NY}
  \country{USA}
}
\email{may13@rpi.edu}

\author{Tyler Derr}
\affiliation{%
  \institution{Vanderbilt University}
  \city{Nashville}
  \state{TN}
  \country{USA}
}
\email{tyler.derr@vanderbilt.edu}

\renewcommand{\shortauthors}{Leyao Wang et al.}

\AtBeginDocument{%
  }

\copyrightyear{2026}
\acmYear{2026}
\setcopyright{cc}
\setcctype{by}
\acmConference[KDD '26]{Proceedings of the 32nd ACM SIGKDD Conference on Knowledge Discovery and Data Mining V.1}{August 09--13, 2026}{Jeju Island, Republic of Korea}
\acmBooktitle{Proceedings of the 32nd ACM SIGKDD Conference on Knowledge Discovery and Data Mining V.1 (KDD '26), August 09--13, 2026, Jeju Island, Republic of Korea}
\acmPrice{}
\acmDOI{10.1145/3770854.3780311}
\acmISBN{979-8-4007-2258-5/2026/08}

\begin{document}

















\begin{abstract}
Real-world graph data often follows long-tailed distributions, making it difficult for Graph Neural Networks (GNNs) to generalize well across both head and tail classes. Recent advances in Vicinal Risk Minimization (VRM) have shown promise in mitigating class imbalance with numeric interpolation; however, existing approaches largely rely on embedding-space arithmetic, which fails to capture the rich semantics inherent in text-attributed graphs.
In this work, we propose our method \textbf{\method} (\textbf{S}emantic-\textbf{a}ware \textbf{V}icinal Risk Minimization for Long-Tailed \textbf{T}ext-\textbf{A}ttributed \textbf{G}raphs), a novel VRM framework that leverages Large Language Models (LLMs) to perform text-level interpolation, generating on-manifold, boundary-enriching synthetic samples for minority classes. To mitigate the risk of noisy generation, we introduce a confidence-based edge assignment mechanism that uses graph topology as a natural filter to ensure structural consistency.
We provide theoretical justification for our method and conduct extensive experiments on benchmark datasets, showing that our approach consistently outperforms both numeric interpolation and prior long-tailed node classification baselines. Our results highlight the importance of integrating semantic and structural signals for balanced and effective learning on text-attributed graphs.  
The source code is publicly available at: \url{https://github.com/LWang-Laura/SaVe-TAG}.
 \vskip -1.3ex
\begin{figure}[h!]
    \centering
    \includegraphics[width=\linewidth]{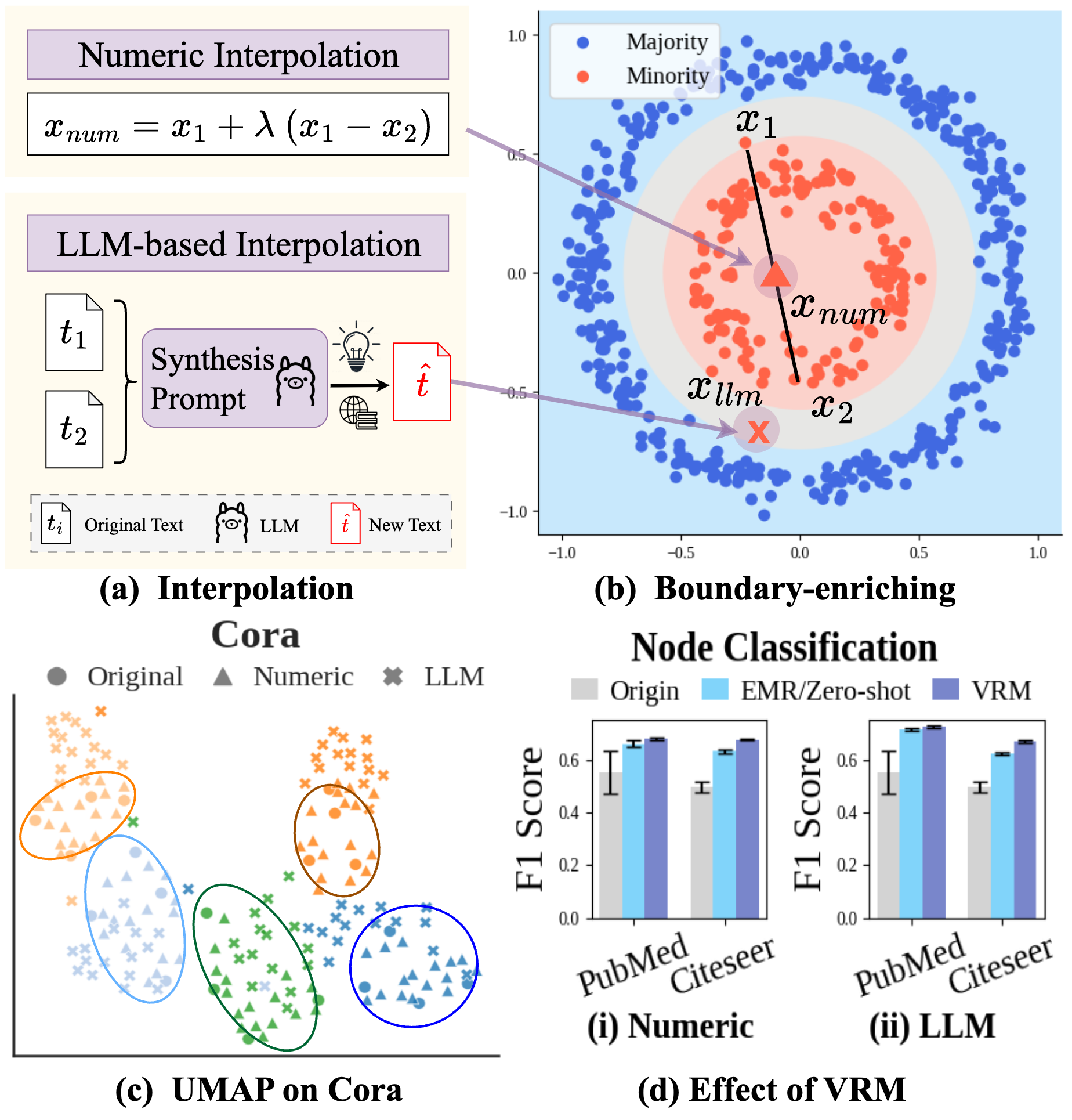}
    \vskip -2ex
  \caption{
Motivation. (a) Numeric vs. LLM-based interpolation.~ (b) LLM-based interpolation enriches decision boundaries (on synthetic data).~ (c) UMAP of original and interpolated samples on Cora; numeric interpolations remain mostly within class boundaries, while LLM-based samples extend beyond.~ (d) Vicinal Risk Minimization (VRM) outperforms Empirical Risk Minimization (ERM) in data augmentation.
}
\vskip -3ex
    \label{fig:fig1}
\end{figure}
 \end{abstract}

\settopmatter{authorsperrow=4}
\maketitle


\section{Introduction}\label{sec-intro}

Graph Neural Networks (GNNs) have proven effective for node classification by modeling structural dependencies~\cite{kipf2016semi,hamilton2017inductive,rong2020deep}. 
However, real-world graphs frequently exhibit long-tail distributions \cite{ma2025class, wang2021distance, 10.1145/3637528.3671880}, making it difficult to learn 
underrepresented classes.
Recent approaches to data imbalance often adopt the Vicinal Risk Minimization (VRM) principle \cite{chapelle2000vicinal}, which minimizes loss over a synthetic neighborhood of the training data—extending 
Empirical Risk Minimization (ERM) that focuses solely on the existing data~\cite{vapnik1998statistical}. 
Typically,
VRM is instantiated via interpolation 
like Mixup \cite{zhang2017mixup} and SMOTE \cite{chawla2002smote}, and has been adapted for long-tailed graph learning~\cite{zhao2021graphsmote,wang2021mixup}. 
However, these methods often apply mathematic operations over input embeddings or hidden features, failing to preserve the rich manifold structure of text-attributed graphs (TAGs).

Advances in large language models (LLMs) present new opportunities for data augmentation by 
offering external knowledge and creative generation. However, existing work on 
LLM-based augmentation typically rely on zero-shot or few-shot prompting~\cite{li-etal-2023-synthetic,yu2024leveraginglargelanguagemodels}, which usually fail to generate samples within the neighborhood of minority instances, as required by VRM principle. Nevertheless, VRM has been shown critical in data augmentation, especially in improving long-tailed node classification. As seen in Figure~\ref{fig:fig1}d (i), VRM-guided numeric interpolation with nearest neighbors outperform ERM-style embedding duplication in node classification.

Motivated by this insight, we hypothesize that prompting strategies explicitly designed to mimic numeric interpolation will surpass standard zero/few-shot approaches, validated by the results in Figure~\ref{fig:fig1}d (ii). Furthermore, integrating controlled prompting into LLM generation can alleviate common pitfalls of numeric interpolation, enabling manifold-aware, boundary-enriching augmentation that effectively exploits the knowledge encoded in LLMs (see Figure \ref{fig:fig1}).

While LLM-based interpolation shows great potential, it also carries the risk of generating noisy samples. To mitigate this, we leverage the inherent structure of the graph as a natural filter. Using a confidence-based edge assignment strategy, noisy samples tend to remain isolated, limiting their influence on downstream tasks. In contrast, synthesized nodes aligned with the original distribution will be connected with the original graph, allowing GNN layers to reinforce their alignment through neighborhood aggregation.

In summary, we introduce a novel interpolation framework, \method, which unifies textual semantics and graph topology in data augmentation. Our method employs LLM-based interpolation to expand the vicinal boundary and reduce vicinal risk, followed by a topology-aware filtering mechanism that preserves alignment within the underlying data distribution.
Our key contributions are: 

\begin{itemize}[leftmargin=1.5em,topsep=1.5pt]
    \item To the best of our knowledge, we present the first semantic-aware VRM framework that employs LLM-based interpolation for node classification on long-tailed text-attributed graphs.
     \item We provide theoretical insights showing that LLM interpolation can generate manifold-preserving, boundary-enriching samples that effectively minimize vicinal risk, while a topology-aware filter helps mitigate noisy generation in the synthesized data.
    \item We conduct experiments showing our method outperforms embedding-based augmentations and long-tail graph baselines, driven by topology-aware filters and LLM semantic benefits. 
\end{itemize}

\footnotetext{The source code: \url{https://github.com/LWang-Laura/SaVe-TAG}.}
\vspace{-7pt}

 \begin{figure*}[t!]
    \centering
    \includegraphics[width=0.98\textwidth]{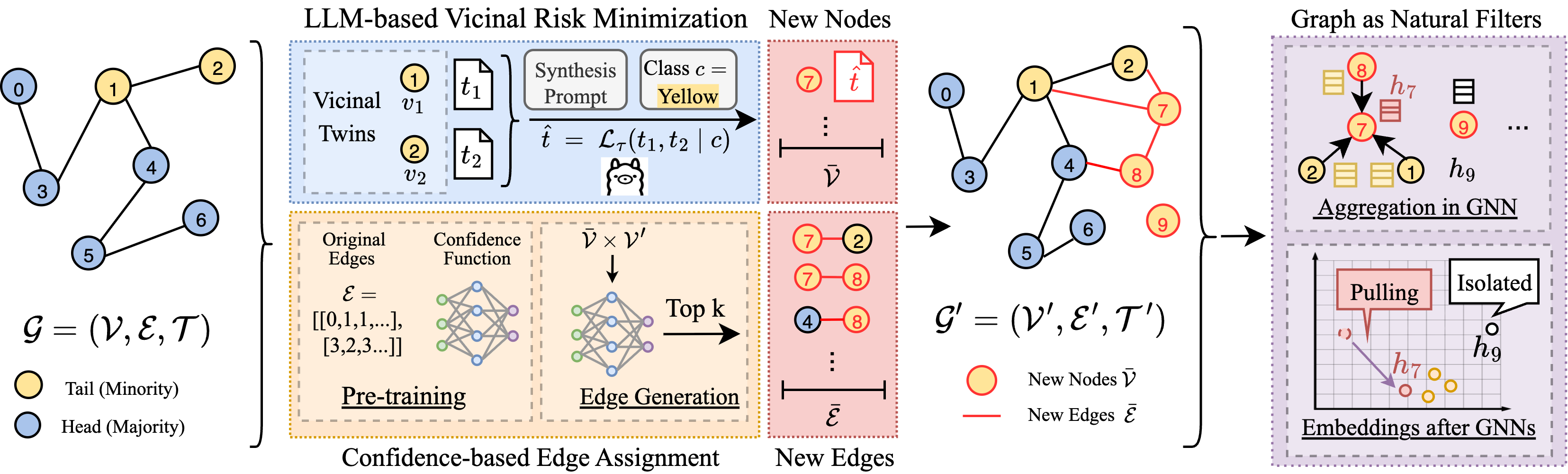}
    \caption{\textbf{An overview of \method.}
Given an input graph \(\mathcal{G} = (\mathcal{V}, \mathcal{E}, \mathcal{T})\), we perform \textit{LLM-based interpolation} on identified \textit{vicinal twins} to synthesize boundary-enriching samples \(\hat{t}\) that minimize vicinal risk.  
A \textit{confidence function} is then pre-trained on the original graph and later used to assign edges to the synthetic nodes via a top-$k$ selection strategy. When the resulting augmented graph \(\mathcal{G}' = (\mathcal{V}', \mathcal{E}', \mathcal{T}')\) is processed by a GNN, such edge assignment incorporates well-aligned nodes into their vicinity while isolating noisy samples.}
    \label{fig-framework}
\end{figure*}


\section{Problem Formulation}
Let $\mathcal{G}=(\V,\E,\T)$ be an attributed graph with node set \(\V\),
edge set \(\E\subseteq \V\times\V\), and per‑node texts \(\T=\{t_v\}_{v\in\V}\). Let $\X$ be the feature matrix, where each row corresponds to a tokenized text $t_v \in \T$. Each node $v$ is associated with a label $y_v \in \{1, \dots, C\}$, where $C$ is the total number of classes, forming the label set $\Y = \{\;y_v \mid v \in \V\;\}$.
Class imbalance induces a skewed empirical distribution on $(\X,\Y)$.
Our goal is to learn a node-level classifier
$f_\theta$  (e.g., a GNN), that
maximizes classification performance.


\section{\method}
\label{sec:framework}
In this section, we present the methodology of \method. We begin with an overview of the main pipeline ($\S$~\ref{ssec:overview}), followed by detailed descriptions of its two core components: \textit{LLM-based Vicinal Risk Minimization} ($\S$~\ref{sec:llm-vrm}) and \textit{Graph as Natural Filters} ($\S$~\ref{sec:graph-natural-filter}).

\subsection{Overview}
\label{ssec:overview}
The overall pipeline of \method{} is illustrated in Figure~\ref{fig-framework}, with the corresponding pseudocode provided in Algorithm 1 (Appendix \ref{appen_sec:details}). Specifically, it comprises the following steps:
\begin{enumerate}[leftmargin=2em,label=\textbf{(\arabic*)}]
    \item \textbf{Vicinal Twin Identification:} identifying same-label \emph{vicinal twins} in tail classes. For each tail-class node $v$, find $k$-NN same-label neighbors, denoted as $\mathcal{N}_k^{(\mathrm{tail})}(v)$.
    \item \textbf{LLM Synthesis:} Synthesize class-consistent texts from vicinal twin pairs using a class-conditioned LLM. Specifically, given node texts $(t_{v_i}, t_{v_j})$ from a vicinal twin pair $(v_i, v_j)$, we query the LLM at temperature $\tau$ (denoted as $\LLM$) to generate a synthetic text $\hat{t}$. The resulting synthetic node $\hat{v}$ is assigned the same label as the original twin, i.e., $y_{\hat{v}} = y_{v_i}$.
    \item \textbf{Confidence-aware Edge Assignment:} Assign confidence-aware edges that \emph{pull} well-aligned samples toward the correct manifold while \emph{isolating} noisy ones. All texts are encoded using a frozen encoder $\phi$. For each newly generated text-attributed node $\hat{v}$, we connect it to high-confidence neighbors identified via a learned scoring function $\kappa(\cdot)$ for confidence.
    \item \textbf{GNN Training:} Train a message-passing GNN on the augmented graph $\mathcal{G}' = (\V', \E', \T')$, where $\V'$ is the augmented node set, $\E'$ is the augmented edge set, and $\T' = \{\; t_v \mid v \in \V' \;\}$ contains the corresponding node texts of $\V'$.
\end{enumerate}
 Steps (1) and (2) are incorporated into \textit{LLM-based Vicinal Risk Minimization} ($\S$~\ref{sec:llm-vrm}), while steps (3) and (4) are unified under \textit{Graph as Natural Filters} ($\S$~\ref{sec:graph-natural-filter}), detailed below.

\subsection{LLM‑based Vicinal Risk Minimization}
\label{sec:llm-vrm}
Here, we formalize how class-conditioned generation preserves semantic manifolds, densifies boundary regions, and improves generalization through Vicinal Risk Minimization (VRM), and present theoretical analyses of these properties. Complete proofs of all theorems are provided in Appendix~\ref{app_sec:add_theory}.

\paragraph{Vicinal Twins for Tail Classes.}
Let $\mathcal{C}_t \subseteq \{1, \ldots, C\}$ denote the set of tail classes. For any text-attributed node $v_1$ with $y_{v_1} \in \mathcal{C}_t$, we define its vicinal twins as
\begin{equation}
\mathcal{N}_{k}^{(\mathrm{tail})}(v_1)
:=
\bigl\{\, v_2 \in \V \;\bigm|\; y_{v_2} = y_{v_1},\, v_2 \in \mathcal{N}_k(v_1) \bigr\},
\end{equation}
that is, same-label neighbors within the $k$-NN vicinity of $v_1$. Synthesizing new nodes from such pairs aligns with the VRM principle, as $(v_1, v_2)$ are close in both label and representation. We refer to such pairs as \emph{vicinal twins}.

\paragraph{LLM‑based interpolation.}

For any pair of vicinal twins $(v_i, v_j)$ and their corresponding text attributes $t_{v_i}$ and $t_{v_j}$, we generate a synthetic text $\hat{t}_{ij}$ by querying the LLM ( denoted as $\mathcal L_\tau$):  $\hat{t}_{ij} = \LLM(t_{v_i}, t_{v_j} \mid y_{v_i})$. This yields a new text-attributed node $\hat{v}_{ij}$ with label $y_{\hat{v}_{ij}} = y_{v_i}$. We collect the set of synthetic nodes as $\bar{\V} := {\hat{v}_{ij}}$ and their corresponding texts as $\bar{\T} := {\hat{t}_{ij}}$.

\subsubsection{Staying on the Manifold}
\label{sec:theory-manifold}
\leavevmode
\vspace{4pt}

Let $\phi:\T\to\X$ denote a frozen text encoder, and write
$x_v:=\phi(t_v)$ for the embedding of node $v$ with text $t_v\in\T$.
For each class $c\in\{1,\dots,C\}$, define the class manifold in
feature space by
\vspace{-6pt}
\[
  \mathcal{M}_c := \{\, x_v \mid y_v = c \,\}, \qquad
  \mathcal{M} := \bigcup_{c=1}^C \mathcal{M}_c .
\]
\vspace{-8pt}

Let $(v_1,v_2)$ be a same‑label pair (\emph{vicinal twins}) with $y_{v_1}=y_{v_2}=c$
and corresponding texts $(t_1,t_2)$.

\begin{theorem}[Off-Manifold Numeric Interpolation]
If $\mathcal{M}_c$ is non-convex, there exist $t_1, t_2$ with $\phi(t_1), \phi(t_2) \in \mathcal{M}_c$ and $\lambda \in (0,1)$ such that
\(
  x_\lambda := \lambda\,\phi(t_1) + (1{-}\lambda)\,\phi(t_2) \notin \mathcal{M}_c
\)
\cite{zhang2018mixupempiricalriskminimization,guo2019mixup,baena2022preventing}.
\end{theorem}

\emph{\textbf{Interpretation.}}
Purely numeric mixup can cross gaps between modes in a curved or multi‑modal
class manifold—common in text semantics—yielding ambiguous training targets.

\begin{theorem}[Manifold-Preserving Class-Consistent Generation]
\label{thm:manifold-preservation}
Let a synthetic text $\hat t \sim \LLM(t_1,t_2 \mid c)$ be a sample from an LLM conditioned on
class $c$. 

Suppose the class‑conditional distribution $\;p\;(\;t\mid y{=}c\;)$ satisfies
$\;\Pr_{t\sim p(\cdot\mid y=c)}\;[\phi(t)\in\mathcal{M}_c] \ge 1-\delta$ (where $\Pr$ denotes probability) for some small error
$\delta\in[0,1)$, and that $\LLM(\cdot\mid c)$ approximates
$p(\cdot\mid y{=}c)$ with total‑variation error at most $\epsilon$.
Then the generated text $\hat t$ will be mapped to the class-$c$ manifold with probability at least $1-(\delta+\epsilon)$, denoted as
$\Pr\big[\phi(\hat t)\in\mathcal{M}_c\big] \;\ge\; 1-(\delta+\epsilon).$
\end{theorem}

\emph{\textbf{Interpretation.}}
Conditioning the generator on class labels ensures that most synthetic texts remain on the corresponding class manifold, motivating \emph{class‑conditioned} prompting in \method{}.

\subsubsection{Boundary‑Enriching Effect of LLM Samples}
\label{sec:boundary}
\leavevmode
\vspace{4pt}

Let $\mathcal{X}\subseteq\mathbb{R}^d$ be the input space and $\mathcal{Y}=\{1,\dots,C\}$ the label set.
A classifier $f$ outputs logits $f^{(c)}(x)$ for class $c\in\mathcal{Y}$ and the logit margin is denoted as $\gamma$.
We augment $\X$ with $m$ LLM‑generated pairs $\{(\hat x_j,\hat y_j)\}_{j=1}^{m}$ to form
${\X}':=\X\cup\{(\hat x_j,\hat y_j)\}_{j=1}^{m}$.
\vspace{1ex}

\begin{definition}[The Minimum Margin]
For a labeled example $(x,y)$, define its (logit) margin $\gamma(x,y)$ to be the difference between the model's score for the true class $y$ and the highest score assigned to any other class, denoted as:
$\gamma(x,y)\;:=\; f^{(y)}(x)\;-\;\max_{c\neq y} f^{(c)}(x),$

\noindent
and for any labeled set $S$, define the minimum margin
$\gamma_{\min}(S)\;:=\min_{(x,y)\in S}\gamma(x,y)$.
\end{definition}

\begin{definition}[Boundary‑Coverage Rate (BCR)]
\label{def:bcr}
A labeled sample $(x,y)$ is a boundary sample if the majority of its $k$ Nearest Neighbors have labels different from $y$ \cite{10.1007/11538059_91}. Boundary‑Coverage Rate (BCR) is the fraction of boundary samples in
$\X'$ .
\end{definition}

\begin{theorem}[Margin Lower Bound]
\label{thm:margin}
Let $\eta > 0$ be a fixed margin slack parameter, ~$\gamma_{\min}(\X)$ denote the minimum margin achieved by a model trained on the original dataset $\X$, and $\gamma_{\min}(\X')$ denote the minimum margin after retraining on the augmented dataset $\X'$. \

Then, adapting the argument from \cite{10.5555/3618408.3619509}, the minimum margin after augmentation and retraining decreases by at most $\eta$ times the fraction of non-boundary samples, that is:
\vspace{-0.5pt}
\[
\gamma_{\min}\bigl(\X'\bigr) \;\ge\; \gamma_{\min}(\X)\;-\;\eta\,\bigl(1-\mathrm{BCR}\bigr).
\]
\end{theorem}

\textbf{\emph{Interpretation.}}
Larger $\mathrm{BCR}$ means more near‑boundary 
samples; the bound shows the margin drop is capped by $\delta(1-\mathrm{BCR})$, so higher $\mathrm{BCR}$ improves robustness near decision boundaries \cite{10.1007/11538059_91}.

\subsubsection{Vicinal Risk Minimization with LLM-based Neighborhood}
\label{sec:vrm}
\leavevmode

Following 
VRM~\cite{chapelle2000vicinal}, we define, for each training text $t_i$ with class label $c_i\!\in\!\{1,\dots,C\}$, an LLM‑based vicinal distribution
\[
\mathcal V_{\LLM}(t_i)
\;=\;
\sum_{j=1}^{n} w_{ij}\,
\Pr\;[\;\hat t = \LLM(t_i,t_j\mid c_i)\;],
\quad
w_{ij}\!\ge\!0,\;\; \sum_{j} w_{ij}=1,
\]
The corresponding vicinal risk is
\[
R_{\mathrm{vrm}}(f_\theta)
\;=\;
\frac{1}{n}\sum_{i=1}^{n}
\mathbb{E}_{\hat t\sim \mathcal V_{\LLM}(t_i)}
\!\left[\,
\mathcal J \; \!\big(f_\theta(\phi(\hat t)),\,c_i\big)
\right],
\]
where $\mathbb{E}$ is the expectation operator, $\phi$ is a (frozen) text encoder, $f_\theta$ is the classifier, and $\mathcal J$ is the per‑example loss (e.g., cross‑entropy).

\begin{theorem}[On‑manifold vicinal risk]
\label{thm:manifold-vrm}
If Theorem~\ref{thm:manifold-preservation} holds, then $R_{\mathrm{vrm}}(f_\theta)$ is evaluated on‑manifold with probability at least $1-\delta$ (over the randomness of LLM sampling).
\end{theorem}

\emph{\textbf{Interpretation.}}
Thm.~\ref{thm:manifold-vrm} ensures VRM mostly trains on valid (on‑manifold) synthetic data. 

\begin{theorem}[Boundary coverage $\Rightarrow$ lower vicinal risk]
\label{thm:margin-vrm}
Let $f_\theta^{\mathrm{aug}}$ denote the model trained by VRM on the augmented dataset $\X'$, and let $f_\theta^{\mathrm{orig}}$ denote the model trained without augmentation on $\X$. Assume the loss function $\mathcal J$ is $L$-Lipschitz in its first argument. Then, for any tolerance parameter $\zeta > 0$,
\[
R_{\mathrm{vrm}}(f_\theta^{\mathrm{aug}})
\;\le\;
R_{\mathrm{vrm}}(f_\theta^{\mathrm{orig}})
\;-\; L\,\gamma_0\,(\mathrm{BCR}-\zeta)
\;+\;\mathcal O(\delta),
\]
where $L$ is the Lipschitz constant of the loss (i.e., its sensitivity to input changes), $\gamma_0 := \gamma_{\min}(\X)$ is the minimum margin before augmentation, and $\mathrm{BCR} \in [0,1]$ denotes the boundary coverage ratio of the synthesized vicinal mass.
\end{theorem}

\emph{\textbf{Interpretation.}}
 A higher BCR (i.e., more boundary samples) leads to a lower vicinal risk. If a sufficient fraction of the vicinal mass covers the decision boundary (i.e., large $\mathrm{BCR}$), the loss is Lipschitz-smooth, and off-manifold error $\delta$ is small, then the vicinal risk after augmentation is guaranteed to decrease— by at least $L\,\gamma_0\,(\mathrm{BCR}-\zeta)$, up to a small error $\mathcal{O}(\delta)$.

\subsection{Graph as a Natural Filter}
\label{sec:graph-natural-filter}
Here, we outline how the graph structure helps to denoise LLM-based VRM. We introduce a confidence function to assign edges between synthetic and original nodes, ensuring reliable anchoring while isolating noisy samples. Theoretical analysis shows that high-confidence edges guide synthetic nodes toward the correct class manifold, while low-confidence samples remain isolated. Complete proofs of all theorems are provided in Appendix~\ref{app_sec:add_theory}.

\paragraph{Notation}
Let $\bar{\V}$ denote the set of synthetic nodes, and let $\bar{\T} = \{\hat t_{\hat v} : \hat v \in \bar{\V}\}$ be their corresponding texts, where each $\hat v$ is a newly generated text-attributed node produced by \textit{LLM-based Vicinal Risk Minimization} (see $\S$~\ref{sec:llm-vrm}). We employ a frozen encoder $\phi : \T \cup \bar{\T} \to \X'$ to obtain representations for all texts.

\subsubsection{Confidence function.}
\leavevmode
\vspace{4pt}

Using the frozen encoder $\phi$, we map each node $v$ (real or synthetic) to a predictor-space embedding via an MLP encoder $\mathrm{enc}$:
\[
z_v := \mathrm{enc}\big(\phi(t_v)\big) \in \R^{p}.
\]
For any candidate edge $(u, v)$, a second MLP (the predictor, $\mathrm{pred}$) maps their inner-product similarity to a logit:
\[
s_{u,v} := \mathrm{pred}\big(\langle z_u, z_v \rangle\big) \in \R.
\]
We train $\mathrm{enc}$ and $\mathrm{pred}$ on the \emph{original} graph to score true edges higher than non-edges using a \texttt{LogSigmoid} loss on $s_{u,v}$. At test time, we convert the logit to a calibrated pairwise \emph{confidence}:
\[
\kappa(u, v) := \sigma(s_{u,v}) \in [0,1], \qquad \sigma(a) = \frac{1}{1 + e^{-a}}.
\]
Ranking by $\kappa$ or by $s$ is equivalent, as $\sigma$ is monotonic.

\subsubsection{Confidence-aware edge assignment.}
\leavevmode
\vspace{4pt}

Let $\bar{\V}$ be the set of synthetic nodes, $\V$ the set of original nodes, and $\V' := \V \cup \bar{\V}$ the complete augmented node set. For each candidate pair $(u, \hat v) \in \V' \times \bar{\V}$, we compute $\kappa(u, \hat v)$.

\vspace{3pt}
\noindent
\textbf{Global top-$K$ selection.}
Construct the score matrix over all pairs $\mathcal{P} := \V' \times \bar{\V}$ and select the global top-$K$:
\vspace{-3pt}
\[
\bar{\E} \;=\; \operatorname*{arg\,topK}_{K,\,(u, \hat v) \in \mathcal{P}} \;\kappa(u, \hat v).
\]
\vspace{-7pt}

\noindent
Here $K = k\,|\bar{\V}|$, yielding an average of $k$ edges per synthetic node. The final augmented graph is 
$\G' \;:=\; (\V',\,\E \cup \bar{\E}),$
which is used for downstream node classification.

\paragraph{\textbf{Interpretation.} Why confidence?}
Noisy generations exist. The learned gate $\kappa$ (calibrated on real edges)
prioritizes reliable anchors: high-agreement neighborhoods \emph{pull} a synthetic node
toward the correct region, while uniformly low agreement leaves it \emph{isolated}, preventing error propagation.

\subsubsection{Denoising in GNN Aggregation.}
\leavevmode
\vspace{4pt}

\emph{Aggregator.}
Adapting \cite{kipf2017semisupervisedclassificationgraphconvolutional,veličković2018graphattentionnetworks}, consider a generic message-passing layer that updates the hidden representation $h^{(\ell)}_v \in \R^{p}$ of node $v$ at layer $\ell$ by blending its own transformed features with a weighted sum of its neighbors’ transformed features:
\begin{equation}
\label{eq:agg}
h^{(\ell+1)}_v \;=\;
\alpha\,\W\,h^{(\ell)}_v
\;+\;(1-\alpha)\!\!\sum_{u\in\mathcal N(v)}\!\beta_{vu}\,\W\,h^{(\ell)}_u,
\end{equation}
where $\alpha \in [0,1]$ is a mixing coefficient, $\W \in \R^{p \times p}$ is a trainable weight matrix, $\mathcal N(v)$ denotes the neighbor set of $v$ in $\G'$, and $\beta_{vu} \ge 0$ are neighbor weights with $\sum_{u \in \mathcal N(v)} \beta_{vu} = 1$.

\paragraph{Representation manifolds.}
After training, each class $c\in\{1,\dots,C\}$ forms a representation manifold $\mathcal M_c'\subset\R^{p}$ in the learned feature space.  
For any set $\mathcal S\subset\R^{p}$, write $\mathrm{dist}(\xi,\mathcal S)=\inf_{y\in\mathcal S}\|\xi-y\|_2$ for the Euclidean distance from a point $\xi$ to the closest point in $\mathcal S$. 


\label{sec:pull}
\begin{theorem}[~Confidence $\Rightarrow$ pulling~]
\label{thm:pull}
Let $\hat v\in\bar{\V}$ be synthetic with label $c$ and define the confidence-filtered anchor set (for threshold $\tau\in(0,1)$)
\[
\mathcal N(\hat v)\;:=\;\{\,u\in\V:\kappa(\hat v, u)\ge\tau\,\}\neq\varnothing.
\]
Assume every neighbor $u\in\mathcal N(\hat v)$ satisfies
$\mathrm{dist}\!\bigl(h^{(\ell)}_u,\mathcal M_c'\bigr)\le \varepsilon$ for some $\varepsilon\ge 0$.
Then there exists a contraction factor $\varrho \in(0,1)$ (independent of $\kappa$) such that
\[
\mathrm{dist}\bigl(h^{(\ell+1)}_{\hat v},\mathcal M_c'\bigr)
\;\le\;
\varrho\;\mathrm{dist}\bigl(h^{(\ell)}_{\hat v},\mathcal M_c'\bigr),
\]
and iterating \eqref{eq:agg} yields exponential convergence toward $\mathcal M_c'$
under standard non-expansiveness conditions on \eqref{eq:agg}
\cite{li2018deeperinsightsgraphconvolutional,Gama_2020,cordonnier2025convergencemessagepassinggraph}.
\end{theorem}

\emph{\textbf{Interpretation.}}
High-confidence anchors act as attractors; message passing contracts the synthetic representation
toward the class-$c$ manifold, mitigating residual generation noise.


\begin{theorem}[~No confidence $\Rightarrow$ isolation~]
\label{thm:isol}
If $\kappa(\hat v, u)<\tau$ for all $u\in\V$, then
$\mathcal N(\hat v)=\varnothing$ and the update reduces to
$h^{(\ell+1)}_{\hat v}=\alpha\,\W\,h^{(\ell)}_{\hat v}$ for all $\ell$.
Thus $\hat v$ neither influences nor is influenced by $\G'$.
\end{theorem}

\emph{\textbf{Interpretation.}}
When agreement is uniformly low, the design quarantines the sample, preserving stability of the original graph.


\section{Empirical Results}
\label{sec:exp}
In this section, we present extensive experiments to validate our proposed method, emphasizing its benefits from both \textit{Semantic} and \textit{Structural} perspectives. We further evaluate the integrated framework to demonstrate its effectiveness in addressing long-tailed node classification. Specifically, our experiments aim to answer the following three research questions (RQs):

\begin{itemize}[leftmargin=1.5em]
    \item \textbf{RQ1:} Can large language models (LLMs) effectively generate boundary-enriching samples with manifold interpolation, enabling VRM for textual data distributions?
    
    \item \textbf{RQ2:} Can our confidence-based edge assignment effectively filter out noisy samples?
    
    \item \textbf{RQ3:} Does the entire proposed framework outperform existing state-of-the-art methods in graph learning on long-tailed TAGs?
\end{itemize}

We first describe our experimental setup, followed by an in-depth analysis to address the stated research questions. Our results consistently show that the proposed method significantly outperforms existing baselines, signifying its effectiveness for node classification on long-tailed TAGs.

\subsection{Experimental Setup}

\subsubsection{Datasets}

Following prior studies on long-tailed graph learning~\cite{Fu_2023, zhou2023graphsrdataaugmentationalgorithm, wang2021mixup}, we use Cora~\cite{mccallum2000automating}, PubMed~\cite{sen2008collective}, and Citeseer~\cite{sen2008collective}, along with three larger Amazon datasets, namely Photo, Computer, and Children~\cite{yan2023comprehensive}, for scalability analysis. 

To realistically simulate long-tailed distribution scenarios typical in real-world graphs, we follow the standard protocol from prior work~\cite{10.1145/3511808.3557381, 10.1145/3637528.3671880, Fu_2023, zhou2023graphsrdataaugmentationalgorithm, wang2021mixup, zhao2021graphsmote}. Specifically, we randomly select 20 nodes from each majority (head) class, and $20 \times \textit{imbalance\_ratio}$ nodes from each minority class. A summary of the detailed statistics of these datasets is provided in Table~\ref{tab-dataset} from Appendix \ref{appen_sce:analysis}.

\subsubsection{Metrics}

Following prior work on long-tailed graph classification~\cite{10.1145/3511808.3557381, 10.1145/3637528.3671880, Fu_2023, zhou2023graphsrdataaugmentationalgorithm, wang2021mixup, zhao2021graphsmote}, we evaluate model performance using four standard metrics: balanced accuracy (bAcc), macro-F1 score (F1), geometric mean (GMean), and traditional accuracy (Acc). All experiments are repeated five times with different fixed seeds, and we report the mean performance across these runs.

To further assess boundary-enriching effects, we report the Boundary Coverage Ratio (BCR), as defined in Definition~\ref{def:bcr}, and the Boundary Proximity Score (BPS). BPS is computed as the reciprocal of the Trust Score~\cite{jiang2018trusttrustclassifier}, measuring the ratio between the distance to the nearest out-of-class centroid and the distance to the in-class centroid.
Additionally, we report the In-Class Rate (ICR), defined as the percentage of samples predicted with the same label as their ground truth. Implementation details for these metrics are provided in Section~\ref{sec:manifold}.

\subsubsection{Baselines}
\label{subsubsec:baselines}


We compare our method against two categories of baselines (with details found in Appendix~\ref{appen_sec:details}): 

\begin{itemize}[leftmargin=1em]
    \item \textit{LLM-based data augmentation}: zero shot and few shot~\cite{li-etal-2023-synthetic,yu2024leveraginglargelanguagemodels}
    
    \item \textit{Long-tailed graph learning}: Oversampling~\cite{chawla2003c4}, SMOTE~\cite{chawla2002smote}, \\Embed-SMOTE\cite{ando2017deep}, MixupForGraph~\cite{wang2021mixup}, GraphSMOTE$_T$\cite{zhao2021graphsmote}, GraphSMOTE$_O$\cite{zhao2021graphsmote}, LTE4G\cite{10.1145/3511808.3557381}, and HierTail \cite{10.1145/3637528.3671880}.
\end{itemize}

To comprehensively evaluate our method, we design and implement the following three variants of \method, each corresponding to a specific VRM-based strategy. Detailed prompting strategies for each variant are provided in the Appendix~\ref{appen_sec:details}.
\begin{itemize}[leftmargin=1em]
    \item \textbf{\method$_O$:} Mirrors the oversampling method but generate a texts similar to the existing one rather than simple duplication from the conventional method.

    \item \textbf{\method$_S$:} Inspired by SMOTE~\cite{chawla2002smote}, interpolates with nearest neighbors within the same class.

    \item \textbf{\method$_M$:} Extends SMOTE~\cite{chawla2002smote} by interpolating embeddings with nearest neighbors from potentially different classes, inspired by Mixup techniques.
\end{itemize}

\subsubsection{Implementation Details}

By default, we generate texts with Llama3-8B-Instruct and use Sentence-BERT (\texttt{all-MiniLM-L6-v2}) as the text encoder $\phi(\cdot)$. Node classification is performed using a Graph Convolutional Network (GCN). See Appendix~\ref{appen_sec:details} for details about hyperparameters. 

\subsection{Semantic-aware Benefits (RQ1)}
\label{sec:manifold}

\paragraph{On-Manifold.}
To compare on-manifold LLM-based and off-manifold numeric interpolation, we evaluate node classification on multiple benchmarks. As shown in Figure~\ref{fig:embed},
\textit{Embed} leverages LLaMA 3.2-1B embeddings for interpolation, while \textit{SimCSE} and \textit{SBERT} serve as alternatives; 
our LLM-based method using LLaMA 3.2-1B consistently outperforms all numeric interpolation baselines. 

\begin{figure}[h!]
    \centering
    \includegraphics[width=0.95\linewidth]{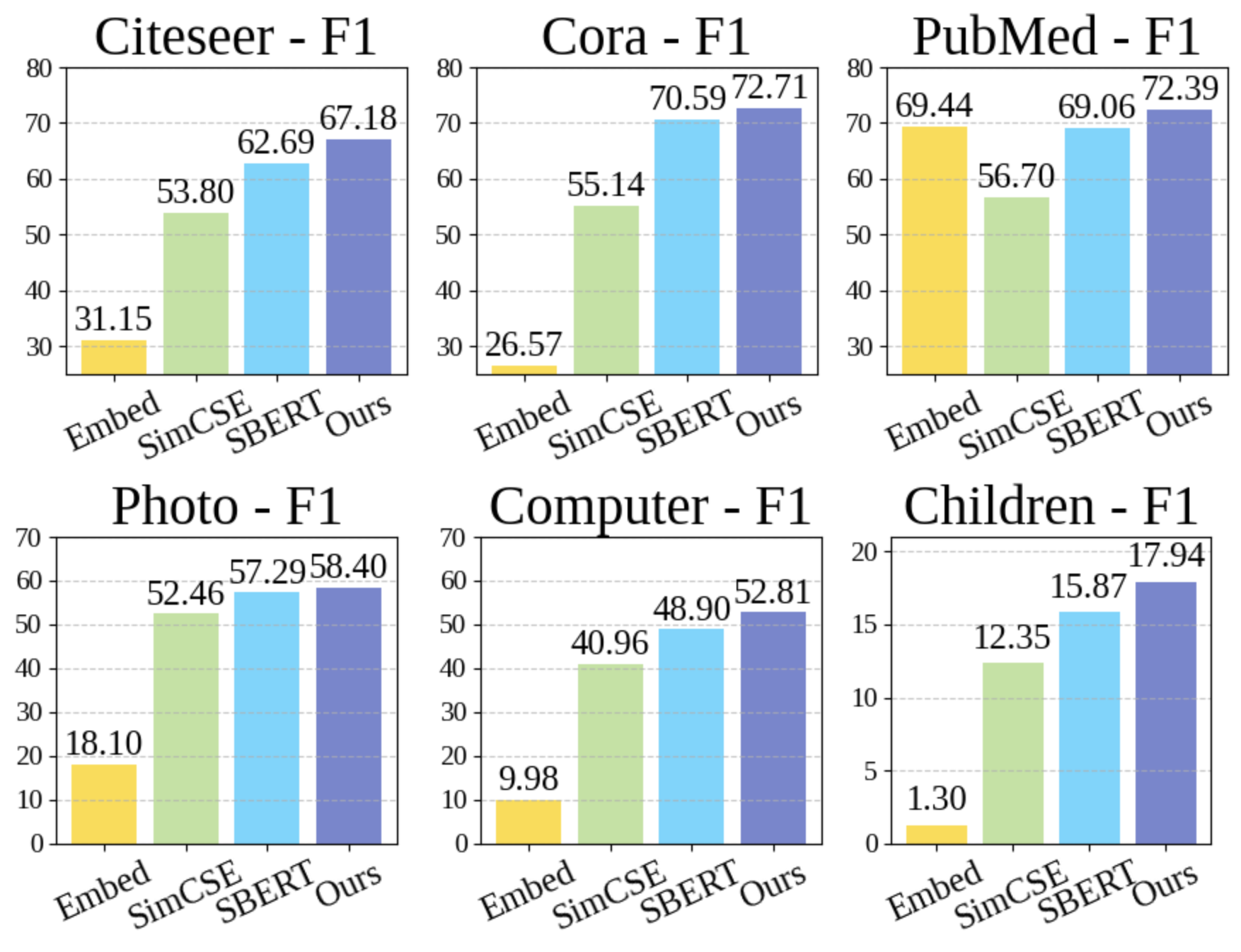}
    \vskip -1ex
   \caption{Node classification performance (F1) under various implementation of \method$_S$. \textit{Embed} uses Llama3.2-1B embeddings for interpolation. \textit{SimCSE} and \textit{SBERT} are compared as alternative embedding baselines. Our LLM-based method with Llama3.2-1B consistently outperforms all the embedding-based baselines. 
   } 
   \vskip -1ex
    \label{fig:embed}
\end{figure}

\paragraph{Boundary-enriching Effects.} We first visualize the embeddings from LLM-based and numeric interpolation. Figure~\ref{fig:umap_all} shows the UMAP projection of all six datasets, with the boundary of the original data circled.
While numeric interpolated samples tend to stay within the original boundary, LLM-generated samples often extend beyond it, enlarging the VRM decision region.

We continue the evaluation with BCR and BPS metrics, as shown in the Figure~\ref{fig:bps_bcr}. These metrics are computed for three settings: the original data, the combined set of original and interpolated data (\texttt{Interp+Orig}), and the interpolated data alone (\texttt{Interp-Only}).
Across all benchmarks, our LLM-based interpolation consistently outperforms numeric interpolation on both metrics, suggesting that LLM-generated samples align more effectively with decision boundaries. Particularly in some cases, LLM-based interpolation even exceeds the original data in the BCR, highlighting its potential to extend and enrich the original decision boundary.

\begin{figure}[t!]
    \centering
    \includegraphics[width=1\linewidth]{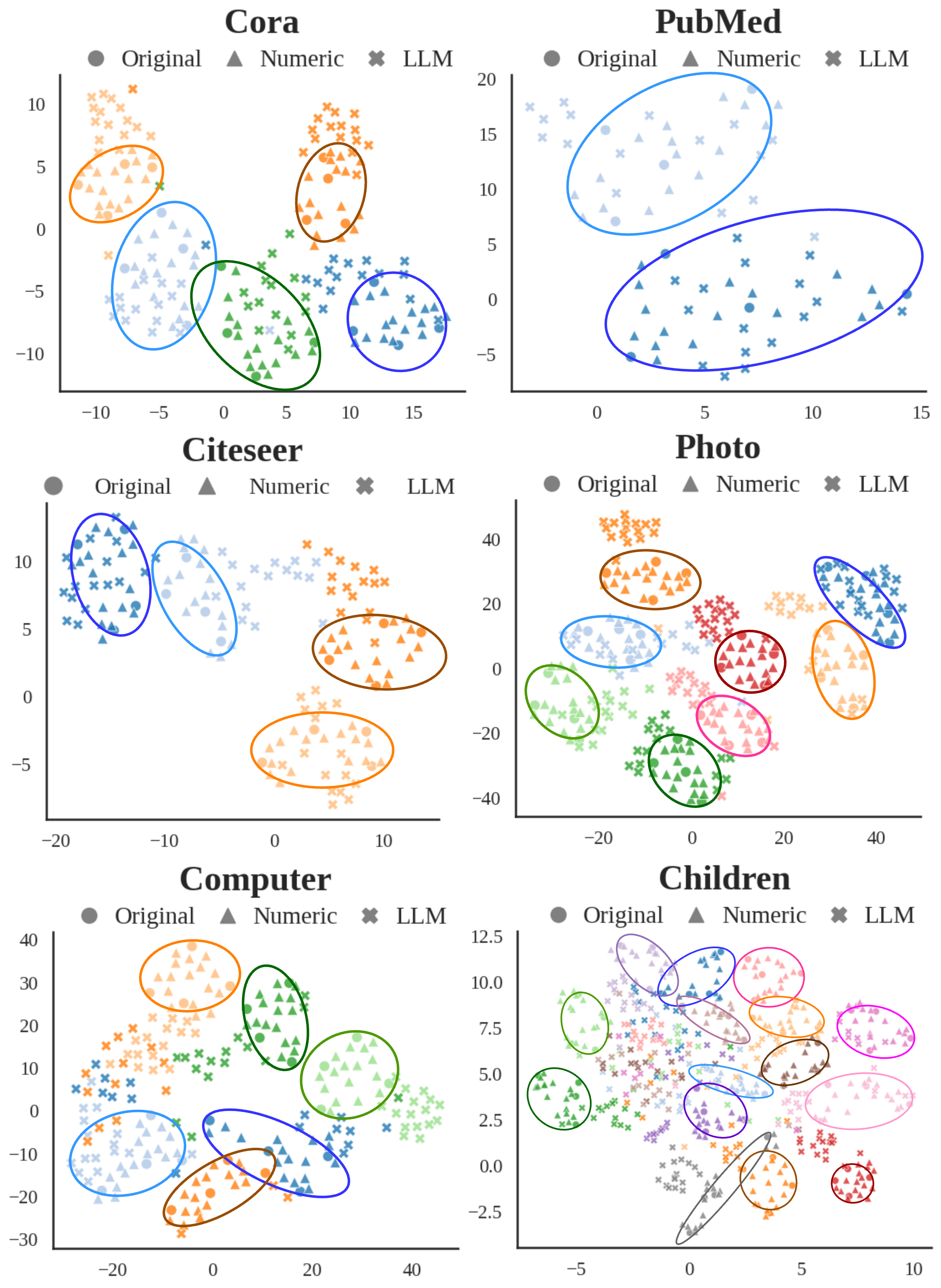}
    \caption{UMAP projection of the original and interpolated samples across various benchmarks---the original data boundary is outlined, with numeric samples mostly staying within and LLM-generated samples extending beyond.}
    \label{fig:umap_all}
     \vskip -1.5ex
\end{figure}
\begin{figure}[!ht]
    \centering
    \includegraphics[width=1\linewidth]{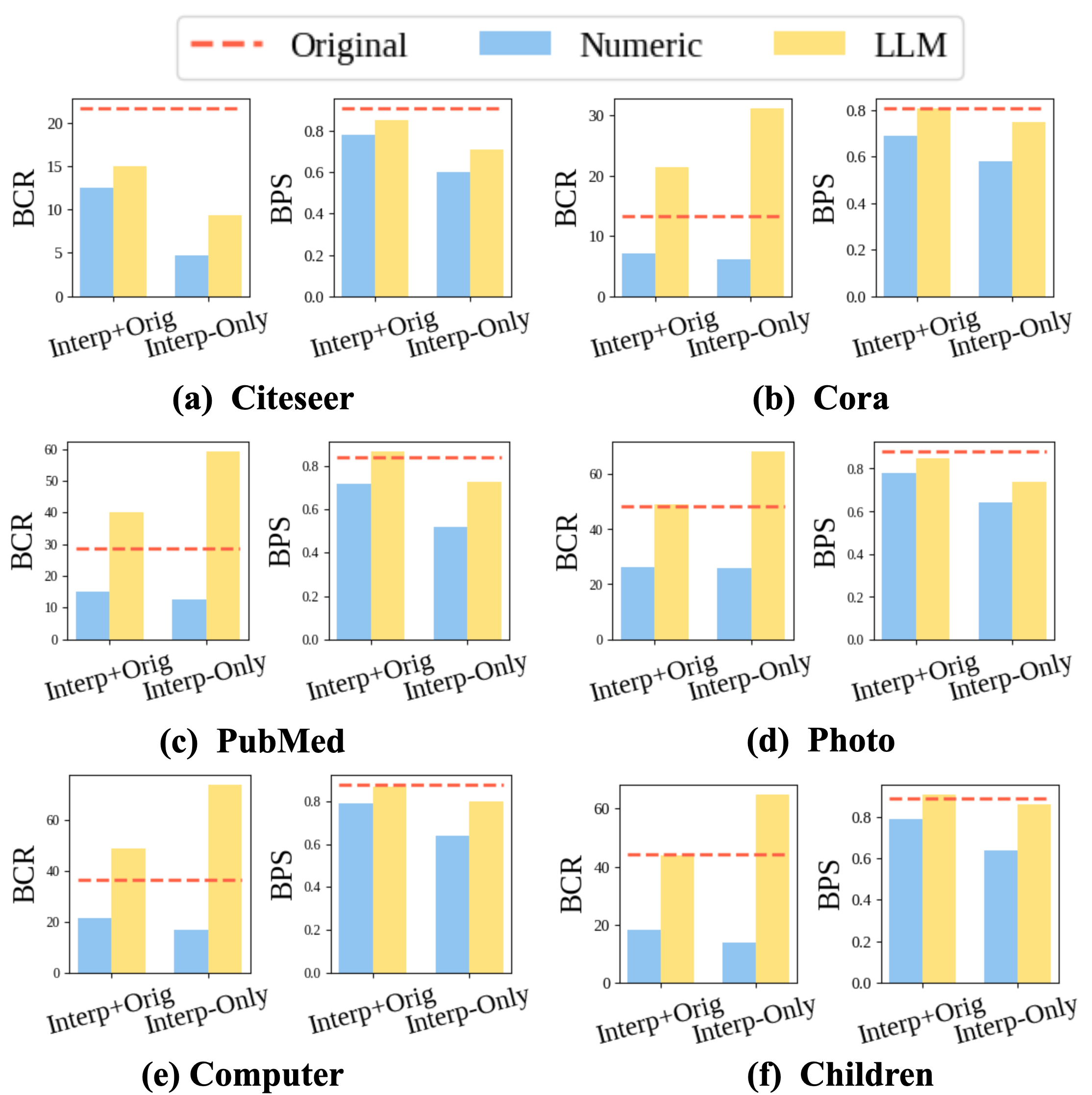}
    \vskip -1.5ex
       \caption{Boundary Coverage Rate (BCR) and Boundary Proximity Score (BPS) for numeric (blue) vs. LLM (yellow) interpolation in Interp+Orig and Interp-Only settings across six datasets; the red dashed line marks the original baseline.}
       \vskip -1ex
    \label{fig:bps_bcr}
\end{figure}

\begin{figure}[h!]
    \includegraphics[width=0.63\columnwidth]
    {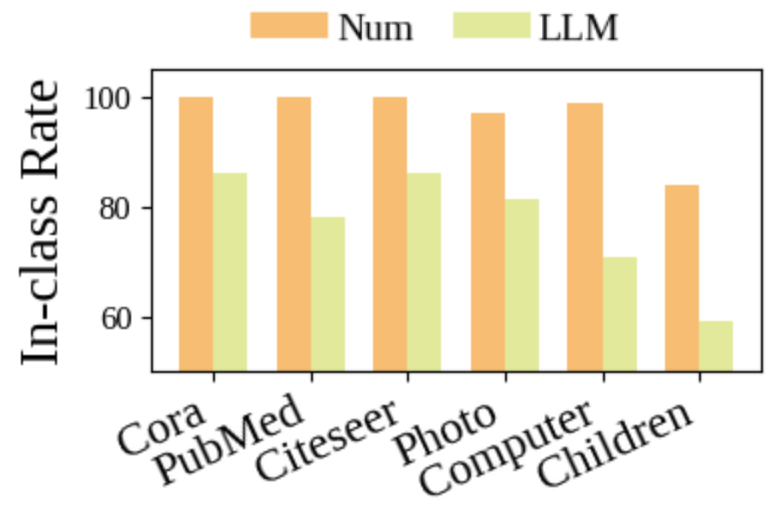}
    \vskip -2ex
    \caption{Average In-Class Rate (ICR) of augmented samples classified by an MLP trained on balanced data. Numerical interpolation achieves near-perfect class consistency (100\% ICR), while LLM-based samples yield lower ICRs, reflecting increased proximity to decision boundaries.}
    \label{fig:icr_all}
    \vskip -2ex
\end{figure}

In addition, we train a separate MLP classifier for each dataset to evaluate the class consistency of the augmented data. Each classifier is trained on a balanced dataset, where all classes have an equal number of training samples. The average In-Class Rate (ICR)—the proportion of augmented samples classified into their intended class—is reported in Figure~\ref{fig:icr_all}. As we can see, Numerical interpolation achieves 100\% ICR on Cora, PubMed, and Citeseer, and nearly 100\% on others, reflecting high class consistency. In contrast, LLM-based samples show lower ICRs, indicating proximity to decision boundaries that causes occasional misclassification by MLP.

\paragraph{LLM as Vicinal Risk Minimization.} We evaluate the effectiveness of LLM-based interpolation within the Vicinal Risk Minimization (VRM) framework. As shown in Table~\ref{tab-lbaseline}, interpolation-based prompting significantly enhances the quality of synthesized textual data, resulting in improved performance on downstream node classification tasks.

\begin{table*}[!t]
\centering
\scriptsize
\caption{Comparison of VRM-based LLM interpolation with standard LLM augmentation. Subscripts $_O$, $_M$, and $_S$ denote the prompting strategies of their corresponding \method~variants, but without edge assignment. Best and second-best results are emboldened and underlined, respectively.}
 \vskip -2ex
\begin{adjustbox}{width=1\textwidth}
\setlength{\tabcolsep}{4pt}
\setlength{\extrarowheight}{-2.5pt}
\begin{tabular}{lccclccclcccl}
\toprule
\textbf{Method} & \multicolumn{4}{c}{\textbf{Cora}}& \multicolumn{4}{c}{\textbf{Pubmed}}& \multicolumn{4}{c}{\textbf{Citeseer}}\\
\cmidrule(lr){2-5} \cmidrule(lr){6-9} \cmidrule(lr){10-13}
 & Acc& F1 &  GMean&bAcc& Acc& F1 & GMean&bAcc& Acc& F1 & GMean&bAcc\\
\midrule
Original& 72.5$\pm$1.4& 70.4$\pm$1.1&  84.5$\pm$0.6&74.7$\pm$0.9&55.1$\pm$8.2& 47.8$\pm$15.0&  62.2$\pm$7.9&51.5$\pm$9.7&  53.7$\pm$0.1& 49.6$\pm$2.2&  71.4$\pm$0.9&55.9$\pm$1.3\\
 Zero-shot& 68.3$\pm$0.5& 63.3$\pm$0.7& 81.1$\pm$0.4&69.4$\pm$0.7& 71.4$\pm$0.4& 71.4$\pm$0.5& 78.3$\pm$0.4&72.3$\pm$0.5& 64.8$\pm$0.5& 62.1$\pm$0.5&78.0$\pm$0.3&65.2$\pm$0.5\\
Few-shots&73.9$\pm$0.6&\textbf{72.7$\pm$0.6}&85.4$\pm$0.2&76.1$\pm$0.3& \underline{72.0$\pm$0.3}& 72.0$\pm$0.3& \underline{79.2$\pm$0.2}&\underline{73.6$\pm$0.3}& 68.2$\pm$0.9& 65.6$\pm$0.7& 79.8$\pm$0.4&67.8$\pm$0.5\\
\midrule
LLM$_O$&73.7$\pm$0.7&72.1$\pm$0.7&85.3$\pm$0.3&76.1$\pm$0.4& 71.5$\pm$0.3& 71.4$\pm$0.3& 78.8$\pm$0.3&73.1$\pm$0.4& \textbf{69.6$\pm$0.8}& \textbf{67.0$\pm$0.6}& \textbf{80.5$\pm$0.4}&\textbf{68.8$\pm$0.5}\\
LLM$_M$& \underline{74.2$\pm$0.4}& \underline{72.6$\pm$0.3}& \textbf{85.6$\pm$0.1}&\textbf{76.5$\pm$0.2}& 71.9$\pm$0.4& \underline{72.0$\pm$0.4}& 79.0$\pm$0.2&73.4$\pm$0.3& \underline{69.5$\pm$0.6}& \underline{66.8$\pm$0.4}& \underline{80.3$\pm$0.2}&\underline{68.6$\pm$0.3}\\
LLM$_S$& \textbf{74.5$\pm$0.4}& 72.5$\pm$0.5& \underline{85.5$\pm$0.2}&\underline{76.3$\pm$0.3}& \textbf{72.5$\pm$0.4}& \textbf{72.5$\pm$0.4}& \textbf{79.4$\pm$0.3}&\textbf{73.8$\pm$0.3}& 69.2$\pm$0.6& 66.5$\pm$0.5& 80.1$\pm$0.3&68.3$\pm$0.5\\
\bottomrule
\end{tabular}
\end{adjustbox}

\label{tab-lbaseline}

\end{table*}

\begin{figure*}
    \includegraphics[width=0.98\linewidth]
    {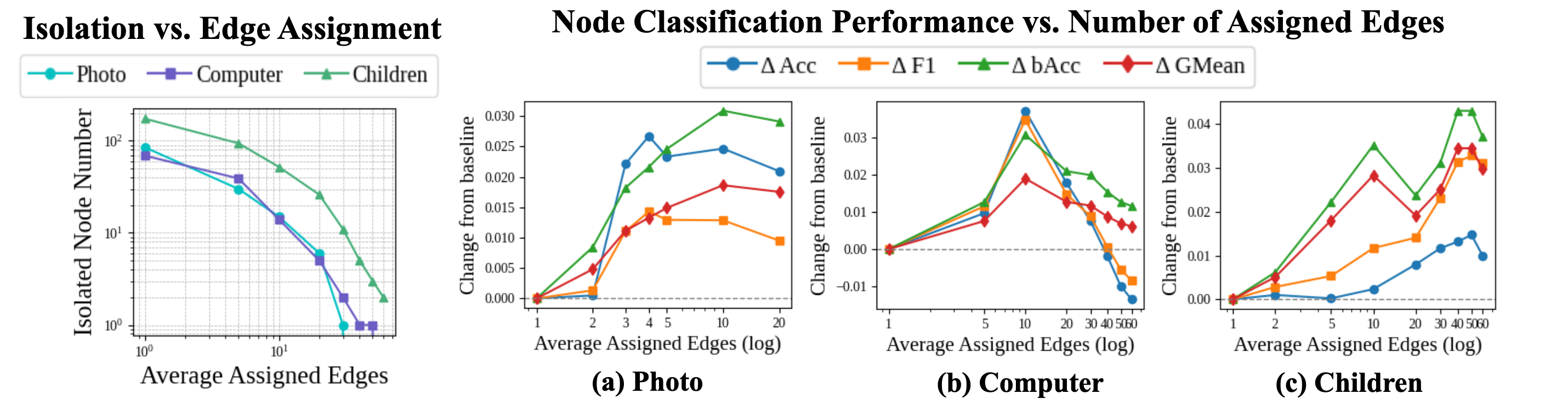}
    \vskip -0.5ex
    \caption{
\textbf{Impact of Number of Assigned Edges.}
\textit{Left:} The number of isolated nodes (log scale) decreases as the average number of assigned edges increases for the Photo, Computer, and Children datasets.
\textit{Right:} Effect of average assigned edges (log scale) on node classification for (a) Photo, (b) Computer, and (c) Children;
performance is reported as changes from the baseline ($n{=}1$).
}
    \label{fig:topk}
\end{figure*}

\subsection{Necessity of Confidence-based Edge Assignment (RQ2)}
\label{sec:edge}

\paragraph{Isolating Noisy Samples.}

 We instantiate Definition~\ref{def:bcr} using a \emph{top–$k$} rule, where each synthetic node retains its $k$ highest-confidence edges. The value of $k$ is set to \(\# augmented\_nodes \times n\), and we vary $n$ to study its effect on node isolation and classification performance.
The left part of Figure~\ref{fig:topk} shows that the number of isolated nodes ($\log$-scaled) drop as $\log(n)$ increases. The right part of Figure~\ref{fig:topk} reveals bell-shaped performance curves with respect to $\log(n)$ on the Photo, Computer, and Children datasets, measured relative to the baselines of $n{=}1$.
These trends support our hypothesis:
\begin{itemize}[leftmargin=1em]
    \item \textbf{Too small $k$ $\rightarrow$ isolation.}
  Many synthetic nodes become singletons
  (Theorem~\ref{thm:isol}), so the minority manifold is
  under-populated and performance suffers.
  \item
  \textbf{Too large $k$ $\rightarrow$ over-smoothing.}
  Excessive connectivity causes repeated averaging in Eq.\eqref{eq:agg}, pushing node representations toward a global mean and resulting in over-smoothing\cite{li2024oversmoothingnightmaregraphcontrastive}.
\end{itemize}

\vspace{-8pt}

\paragraph{Effectiveness of Confidence Function.}

We compare two edge assignment strategies for augmented nodes: edge duplication from the target node (without the ‘${C}$’ subscript) and edge assignment with our confidence-based function (denoted by ‘${C}$’).
Table~\ref{tab:gcn} shows that our confidence-based edge assignment consistently outperforms naive edge duplication under LLM-based interpolation. A similar trend holds for numeric interpolation, with the exception of Cora and Photo, likely due to the generation of off-manifold samples that our confidence mechanism successfully isolates.

\begin{table*}[!t]
\centering
\caption{Comparison of \method ~with previous long-tail graph learning baselines.}
 \vskip -1.5ex
\begin{adjustbox}{width=1\textwidth}
 \setlength{\tabcolsep}{4pt}
 \setlength{\extrarowheight}{-2.5pt}
\begin{tabular}{lccclccclcccl}
\toprule
\textbf{Method} & \multicolumn{4}{c}{\textbf{Cora}}& \multicolumn{4}{c}{\textbf{Pubmed}}& \multicolumn{4}{c}{\textbf{Citeseer}}\\
\cmidrule(lr){2-5} \cmidrule(lr){6-9} \cmidrule(lr){10-13}
 & Acc& F1 & GMean&bAcc& Acc& F1 & GMean&bAcc& Acc& F1 & GMean&bAcc\\
\midrule

 SMOTE& 71.8$\pm$2.7& 70.8$\pm$2.1& 82.7$\pm$1.7& 71.8$\pm$2.7& 68.5$\pm$6.2& 65.5$\pm$10.1& 75.9$\pm$4.8& 68.5$\pm$6.2& 50.7$\pm$3.0& 47.9$\pm$3.1& 67.6$\pm$2.2&50.7$\pm$3.0\\
   Oversampling& 72.9$\pm$1.3& 72.0$\pm$0.9&  83.4$\pm$0.9&72.9$\pm$1.3& 58.5$\pm$6.3& 50.4$\pm$6.5&  68.1$\pm$5.1&58.5$\pm$6.3& 52.0$\pm$4.9& 49.6$\pm$4.7& 68.5$\pm$3.6&52.0$\pm$4.9\\
 EmbedSMOTE& 70.0$\pm$1.6& 69.8$\pm$1.4& 81.5$\pm$1.0& 70.0$\pm$1.6& 63.4$\pm$7.6& 60.3$\pm$9.7& 72.0$\pm$5.9& 63.4$\pm$7.6& 47.5$\pm$2.0& 45.4$\pm$2.5& 65.2$\pm$1.5&47.5$\pm$2.0\\
 GraphSMOTE$_T$& 74.4$\pm$4.0& \textbf{74.2$\pm$4.3}& 84.4$\pm$2.6& 74.4$\pm$4.0& 67.9$\pm$5.0& 67.8$\pm$5.0& 75.5$\pm$3.9& 67.9$\pm$5.0& 50.5$\pm$2.8& 48.1$\pm$3.2& 67.4$\pm$2.1&50.5$\pm$2.8\\
GraphSMOTE$_O$& 74.3$\pm$3.7& 74.1$\pm$4.9&  84.3$\pm$3.0&74.3$\pm$4.7&69.7$\pm$5.0& 69.9$\pm$4.8&  76.9$\pm$3.9&69.7$\pm$5.0&  52.0$\pm$1.6& 49.8$\pm$1.8&  68.6$\pm$1.2&52.0$\pm$1.6\\
MixupForGraph &66.4$\pm$2.7&63.32$\pm$2.7&80.1$\pm$1.6&68.0$\pm$2.5& 69.6$\pm$0.8& 68.7$\pm$1.1& 75.9$\pm$1.1&68.7$\pm$1.6& 63.2$\pm$1.9& 59.4$\pm$1.9& 75.1$\pm$1.8&60.9$\pm$2.5\\
 HierTail \footnotemark[1] & 72.8$\pm$5.3& 73.3$\pm$4.7& 83.3$\pm$3.4& 72.8$\pm$5.3& 72.4$\pm$8.2& 72.4$\pm$7.7& 79.0$\pm$6.3& 72.5$\pm$8.0& 49.6$\pm$9.6& 44.5$\pm$11.3& 66.5$\pm$7.4&49.6$\pm$9.6\\
 LTE4G  \footnotemark[1]
 & 74.2$\pm$3.8& 74.1$\pm$3.8& 84.3$\pm$2.4& 74.2$\pm$3.8& 72.9$\pm$1.5& 72.6$\pm$1.9& 79.4$\pm$1.2& 72.9$\pm$1.5& 51.4$\pm$3.8& 49.2$\pm$3.7& 68.1$\pm$2.8&51.4$\pm$3.8\\

 \midrule
\method$_{O}$& 72.1$\pm$0.1& 70.9$\pm$0.1& 84.5$\pm$0.3&74.8$\pm$0.4& 72.8$\pm$0.3& 73.0$\pm$0.4& 80.2$\pm$0.2&75.0$\pm$0.2& 69.9$\pm$0.2& 66.8$\pm$0.2& 79.9$\pm$0.1&67.8$\pm$0.2\\

\method$_{M}$ & 73.0$\pm$0.4& 71.8$\pm$0.3& 85.1$\pm$0.2&75.8$\pm$0.4& 71.3$\pm$0.2& 71.2$\pm$0.3& 79.1$\pm$0.1&73.6$\pm$0.3&69.5$\pm$0.3&65.9$\pm$0.4&79.3$\pm$0.3&66.9$\pm$0.5\\

\method$_{S}$  & \textbf{76.4$\pm$0.3}& 74.1$\pm$0.3& \textbf{85.8$\pm$0.2}&\textbf{76.6$\pm$0.3}& \textbf{76.3$\pm$0.7}& \textbf{76.4$\pm$0.8}& \textbf{82.5$\pm$0.5}&\textbf{77.6$\pm$0.7}& \textbf{ 69.9$\pm$0.3}&\textbf{ 66.9$\pm$0.3}&\textbf{80.1$\pm$0.3}&\textbf{68.1$\pm$0.4}\\
\bottomrule
\end{tabular}
\end{adjustbox}


\label{tab-gbaseline}

\begin{tablenotes}
    \centering 
      \footnotesize
      \item \footnotemark[1] Modified code while running on PubMed dataset to prevent out-of-memory issues due to scalability.
\end{tablenotes}
\end{table*}


\subsection{Enhancement in Long-tailed Graph Learning (RQ3)}

\begin{figure*}[t!]
    \includegraphics[width=0.98\textwidth]
    {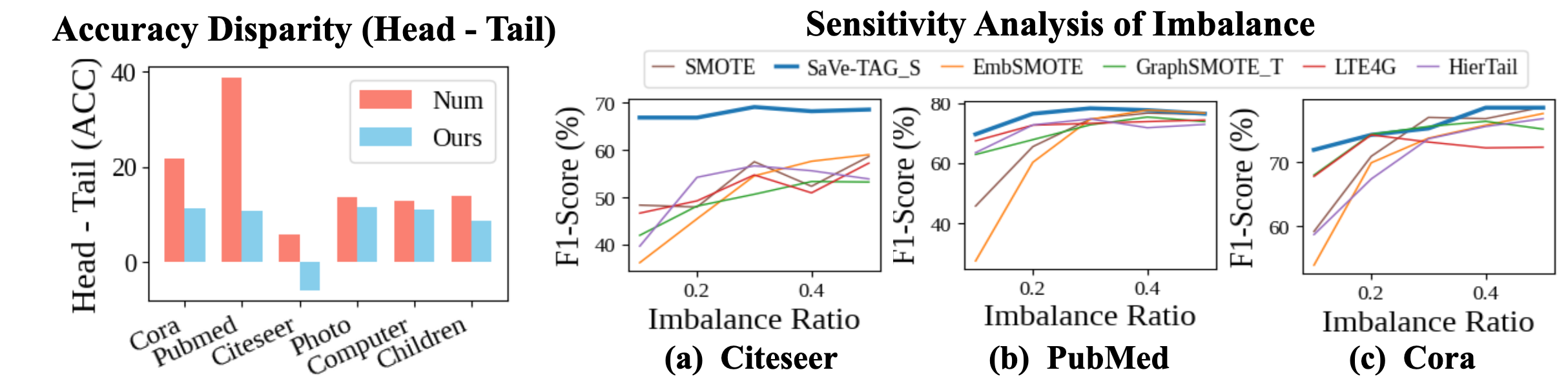}
    \vskip -0.75ex
    \caption{Node Classification Enhancement. \textit{Left:} \method$_S$ achieves a smaller \texttt{Head – Tail} accuracy gap than numeric interpolation, indicating better class balance.\textit{ Right:} \method$_S$ (bold) maintains more stable performance with less fluctuation across imbalance ratios.}
    \label{fig:lt}
    \vskip -0.75ex
\end{figure*}

\begin{table}[t!]
\footnotesize
\setlength{\tabcolsep}{4pt}
 \setlength{\extrarowheight}{-1pt}
\centering
\caption{\textbf{Ablation Studies.} Node classification under three variants of \textsc{SaVe-TAG}. \textit{Origin} uses no augmentation; \textit{Num} and \textit{LLM} apply numeric and LLM-based interpolation, respectively. Subscript $_C$ indicates the use of confidence-based edge generation versus naive edge duplication otherwise.}
\vskip -1ex
\begin{adjustbox}{width=1\linewidth}
\begin{tabular}{@{}llllllll}
\toprule
\multirow{2}{*}{\textbf{Dataset}}  & \multirow{2}{*}{\textbf{Method}} & \multicolumn{2}{c}{\textbf{\method $_S$}}& \multicolumn{2}{c}{\textbf{\method $_M$}}& \multicolumn{2}{c}{\textbf{\method $_O$}}\\ 
\cmidrule(lr){3-4} \cmidrule(lr){5-6} \cmidrule(lr){7-8}
                          &                         
                          & F1 & ACC & F1 & ACC & F1 & ACC \\ 
\midrule
\multirow{5}{*}{\shortstack{Average \\ across \\ all 6 \\ datasets}}
 & Origin   & 45.00 & 47.55 & 45.00 & 47.55 & 45.00 & 47.55 \\
 & Num      & 54.07 & 55.78 & 55.09 & 56.64 & 54.32 & 56.09 \\
 & Num$_C$  & 54.86 & 56.57 & 56.37 & 57.54 & 54.35 & 55.64 \\
 & LLM      & \underline{56.79} & \underline{58.49} &
              \underline{56.90} & \underline{58.18} &
              \underline{56.39} & \underline{58.32} \\
 & LLM$_C$  & \textbf{58.85} & \textbf{61.03} &
              \textbf{59.02} & \textbf{61.74} &
              \textbf{59.24} & \textbf{62.05} \\
\bottomrule
\end{tabular}

\end{adjustbox}

\label{tab:gcn}
\vskip -1ex
\end{table}

\paragraph{Performance.}
Table~\ref{tab-gbaseline} shows that our method, particularly \method $_S$,  outperforms prior baselines in node classification on long-tailed text-attributed graphs.
Moreover, our method effectively narrows the disparity between the head (majority) and tail (minority) classes. As shown in the left part of Figure~\ref{fig:lt}, across all benchmarks, \method$_S$ consistently achieves a smaller disparity (the average accuracy gap between head and tail classes) than numeric interpolation, underscoring its effectiveness in mitigating class imbalance in long-tailed learning.

\vspace{-2pt}
\paragraph{Stability.}
We further perform a sensitivity analysis to evaluate the stability of \method~ under varying imbalance ratios. As shown in the right part of Figure~\ref{fig:lt}, \method $_S$ exhibits smaller fluctuations compared to other baselines as imbalance ratios vary.

\subsection{Ablation Studies}
We conduct the ablation studies to show the significance of both LLM-based interpolation and confidence-based edge assignment, shown in Table~\ref{tab:gcn}.

\textbf{LLM v.s. Num. }Leveraging the generative capabilities of LLMs for VRM proves to be more effective than embedding interpolation. For instance, \method$_{S}$ achieves a 10.3\% increase in average accuracy on the PubMed dataset compared to its numeric counterpart. 

\textbf{w/ $_C$, v.s. w/o $_C$.} With confidence function (denoted as $_C$) enhances the node classification performance. Same analysis has been discussed in Section \ref{sec:edge}

\section{Discussion}
\label{sec:dis}
\paragraph{Related Work.} 
Here we outline relevant research topics. 
\textbf{Vicinal Risk Minimization (VRM):  }VRM~\cite{10.5555/3008751.3008809} extends ERM \cite{vapnik1998statistical} by estimating risk from local vicinity distributions, enhancing model generalization. It underlies methods like Mixup~\cite{zhang2018mixupempiricalriskminimization} and SMOTE~\cite{chawla2002smote}, which create synthetic samples via interpolation. \textbf{Long-tailed Graph Learning:} In graph learning, long-tailed node classification is typically addressed via hierarchical task grouping like HierTail~\cite{10.1145/3637528.3671880} and expert models like LTE4G~\cite{10.1145/3511808.3557381}. \textbf{Large Language Model (LLM) Data Augmentation:} Generally, LLMs augment data by generating contextually relevant text via zero-shot or few-shot instruction-based prompts~\cite{li2024empoweringlargelanguagemodels,li-etal-2023-synthetic}.

\paragraph{Limitations.}
While LLM-based interpolation offers clear benefits, it is more resource-intensive than numeric embedding methods. We include a cost analysis in Appendix~\ref{appen_sce:analysis}. Despite the overhead, the trade-off is worthwhile: our framework supports small open-source language models, needs no extensive training/tuning, and applies LLMs only once during augmentation—allowing reuse across tasks. Most importantly, it consistently outperforms baselines, confirming the practical value of LLM interpolation. 

Another limitation is our exclusive focus on GNN backbones, overlooking non-graph alternatives such as MLPs and transformer-based encoders. This focus enabled a deeper exploration of topology-aware filtering, but limits the generalizability of our findings to non-graph models. To partially address this, Appendix~\ref{appen_sec:experiments} reports results with an MLP, which show similar trends and suggest that our method may generalize beyond GNNs or graph-structured data.

\section{Conclusion}
This paper presents a novel interpolation framework for long-tailed node classification that unifies semantic and structural signals. With LLMs, we generate manifold-consistent, boundary-enriching samples to extend VRM to the textual domain. To mitigate out-of-distribution generation, we introduce a confidence-based edge assignment that filters synthetic nodes via graph connectivity. Theoretically and empirically, our method outperforms numeric interpolation, standard LLM augmentation, and prior long-tailed graph learning across benchmarks.

\section*{Acknowledgements} This research is supported by the National Science Foundation (NSF) under grant numbers IIS2239881 and ECCS2325417.

\clearpage

\bibliographystyle{ACM-Reference-Format}
\balance
\bibliography{ref}

\newpage
\appendix


\newpage


\section{Additional Experimental Results}\label{appen_sec:experiments}

\paragraph{Classifier Variants.}
To assess the generalization of our method across different classifier architectures, we conduct additional ablation studies using one GNN-based backbone (SAGE) and one non-graph-based model (MLP). Since MLP does not incorporate structural information (i.e., it lacks neighborhood aggregation), the confidence-based edge assignment has no effect. Therefore, we report only the results without the subscript $_C$. Table~\ref{tab:sage_mlp} shows \method~ effectively addresses long-tailed classification with both GNNs and MLP; 
its strong performance with MLP highlights its potential for text classification and broader applications.

\vskip -1ex
\begin{table}[h!]
\caption{\textbf{Ablation Studies.} Node classification performance using SAGE and MLP under three variants of \method. }
\vspace{-2ex}
\scriptsize
\setlength{\tabcolsep}{1pt}
\begin{tabular}{lll|cc|cc|cc|cc|cc|cc}

\toprule
 
 & \multirow{2}{*}{\textbf{Model}} 
 &\multirow{2}{*}{\textbf{Method}} 
& \multicolumn{2}{c|}{\textbf{Cora}} 
& \multicolumn{2}{c|}{\textbf{PubMed}}
& \multicolumn{2}{c|}{\textbf{Citeseer}} 
& \multicolumn{2}{c|}{\textbf{Photo}} 
& \multicolumn{2}{c|}{\textbf{Computer}}
& \multicolumn{2}{c}{\textbf{Children}} \\
\cmidrule(lr){4-5} 
\cmidrule(lr){6-7} 
\cmidrule(lr){8-9} 
\cmidrule(lr){10-11} 
\cmidrule(lr){12-13} 
\cmidrule(lr){14-15} 

&& & \textbf{F1} & \textbf{Acc} 
  & \textbf{F1} & \textbf{Acc} 
  & \textbf{F1} & \textbf{Acc} 
  & \textbf{F1} & \textbf{Acc} 
  & \textbf{F1} & \textbf{Acc} 
  & \textbf{F1} & \textbf{Acc} \\
\midrule

 \multirow{6}{*}{\rotatebox{90}{\textbf{\method$_S$}\hspace{5ex}}}&\multirow{2}{*}{\textbf{MLP}}
  &Origin& 50.31 & 52.34 & 25.51 & 42.59 & 61.70 & 64.64 &  4.25 &  7.48 &  3.09 &  4.61 &  1.06 &  2.03 \\
 && Num& 58.28& 61.35& 64.63& 68.62& 61.70 & 64.64 & 34.96& 34.74& 28.19& 31.77& 15.31 & 14.85 \\
 && LLM& \textbf{64.13}& \textbf{66.80}& \textbf{72.77}& \textbf{72.07}& \textbf{64.04}& \textbf{66.25}& \textbf{43.26}& \textbf{40.79}& \textbf{37.65}& \textbf{40.26}& \textbf{16.64}& \textbf{16.83}\\

 \cmidrule(lr){2-15}
 &\multirow{4}{*}{\textbf{SAGE}}&Origin&67.5& 70.2& 60.7& 62.4& 52.9& 55.5& 47.4& 44.2& 36.5& 38.0& 2.5& 4.0\\
 && Num& 70.6& 73.1& 68.3& 68.6& 65.0& 68.0& 54.2& 52.1& 43.1& 47.3& 18.7& 18.6\\
 && Num$_C$& 70.3& 72.4& 69.2& 69.6& 65.8& \underline{69.4}& 50.6& 48.9& 49.3& 54.7& 19.4& 18.7\\
 && LLM& \underline{72.5}& \underline{74.5}& \underline{72.5}& \underline{72.5}& \textbf{67.4}& \textbf{70.0}& \textbf{57.7}& \underline{56.0}& \underline{48.6}& \underline{53.7}& \underline{20.1}& \textbf{20.8}\\
 && LLM$_C$& \textbf{74.0}& \textbf{76.4}& \textbf{76.4}& \textbf{76.3}& \underline{66.6}& 69.3& \underline{57.6}& \textbf{56.1}& \textbf{52.0}& \textbf{59.1}& \textbf{20.3}& \underline{20.3}\\
\midrule
 
 \multirow{6}{*}{\rotatebox{90}{\textbf{\method$_O$}\hspace{5ex}}}&\multirow{2}{*}{\textbf{MLP}}
   &Origin& 50.31 & 52.34 & 25.51 & 42.59 & 61.70 & 64.64 &  4.25 &  7.48 &  3.09 &  4.61 &  1.06 &  2.03 \\
 && Num& 57.85& 60.97& 66.06& 66.46& 60.23 & 63.38 & 35.47& 35.01& 28.65& 32.32& 15.89& 15.32\\
 && LLM& \textbf{63.02}& \textbf{64.12}& \textbf{70.25}& \textbf{69.88}& \textbf{64.68}& \textbf{67.07}& \textbf{43.34}& \textbf{42.15}& \textbf{38.18}& \textbf{42.74}& \textbf{18.65}& \textbf{20.38}\\
 
  \cmidrule(lr){2-15}
  &\multirow{4}{*}{\textbf{SAGE}}&Origin&67.5& 70.2& 60.7& 62.4& 52.9& 55.5& 47.4& 44.2& 36.5& 38.0& 2.5& 4.0\\
  && Num& 70.6& 73.8& 69.0& 69.3& 65.3& 68.6& 54.6& 53.5& 43.0& 47.1& 18.5& 19.0\\
 && Num$_C$& 69.1& \underline{70.9}& 68.3& 68.6& 65.5& 68.9& 49.0& 46.1& 42.9& 46.0& 18.6& 19.1\\
 && LLM& \textbf{72.1}& \textbf{73.7}& \underline{71.4}& \underline{71.6}& \textbf{67.5}& \textbf{70.3}& \underline{57.8}& \underline{56.1}& \underline{48.1}& \underline{53.9}& \underline{20.9}& \underline{21.8}\\
 && LLM$_C$& \underline{69.2}& 70.1& \textbf{72.5}& \textbf{72.3}& \underline{66.2}& \underline{69.0}& \textbf{58.0}& \textbf{57.3}& \textbf{52.5}& \textbf{59.9}& \textbf{22.5}& \textbf{23.8}\\
 
  \midrule
  \multirow{6}{*}{\rotatebox{90}{\textbf{\method$_M$}\hspace{5ex}}}&\multirow{2}{*}{\textbf{MLP}}
     &Origin& 50.31 & 52.34 & 25.51 & 42.59 & 61.70 & 64.64 &  4.25 &  7.48 &  3.09 &  4.61 &  1.06 &  2.03 \\
 && Num& 61.66& 62.89& 67.88& 67.77& 61.06 & 63.57 & 36.98& 35.39& 29.25& 33.24& 15.95& 15.95\\
 && LLM& \textbf{62.94}& \textbf{64.47}& \textbf{71.12}& \textbf{70.56}& \textbf{64.52}& \textbf{66.64}& \textbf{44.84}& \textbf{42.28}& \textbf{38.47}& \textbf{41.47}& \textbf{18.80}& \textbf{20.09}\\

  \cmidrule(lr){2-15}

 &\multirow{4}{*}{\textbf{SAGE}}&Origin&67.5& 70.2& 60.7& 62.4& 52.9& 55.5& 47.4& 44.2& 36.5& 38.0& 2.5& 4.0\\
 && Num& 71.9& 74.0& 69.8& 70.0& 64.8& 68.7& 56.0& 54.1& 46.0& 51.4& 19.4& 19.8\\

 && Num$_C$&  69.8& 70.1& 69.9& 69.7&  64.3& 68.6& 53.6&  51.7&  46.6& 51.4&  19.2&  19.6\\
 && LLM& \textbf{72.3}& \textbf{74.2}& \underline{71.5}& \underline{71.3}& \textbf{68.0}& \underline{70.8}& \underline{58.4}& \underline{56.7}& \underline{49.1}& \underline{54.8}& \underline{21.7}& \underline{23.1}\\
 && LLM$_C$& \underline{70.5}& \underline{72.5}& \textbf{75.7}& \textbf{75.2}& \underline{67.8}& \textbf{70.9}& \textbf{58.6}& \textbf{57.8}& \textbf{50.6}& \textbf{57.4}& \textbf{23.9}& \textbf{26.5}\\
\bottomrule
\end{tabular}

\begin{tablenotes}
    \centering 
      \footnotesize
      \item \textit{Origin} denotes no augmentation; \textit{Num} and \textit{LLM} apply numeric and LLM-based interpolation, respectively. The subscript $_C$ indicates confidence-based edge generation and its absence denotes naive edge duplication. The best results for both MLP/SAGE are in bold, and the second-best results for SAGE are underlined.
\end{tablenotes}

\vskip -7ex
\label{tab:sage_mlp}
\end{table}

\begin{table}[!h]
\centering
\scriptsize
\caption{Performance with additional LLM variants.}
\vskip -3ex
\setlength{\tabcolsep}{1pt}
\setlength{\extrarowheight}{-2.5pt}
\begin{tabular}{lccclccclcccl}
\toprule
\textbf{Model} & \multicolumn{4}{c}{\textbf{Cora}} &
                 \multicolumn{4}{c}{\textbf{PubMed}} &
                 \multicolumn{4}{c}{\textbf{Citeseer}} \\
\cmidrule(lr){2-5} 
\cmidrule(lr){6-9} 
\cmidrule(lr){10-13}
& Acc & F1 & GMean & bAcc & Acc & F1 & GMean & bAcc & Acc & F1 & GMean & bAcc \\
\midrule
Qwen2.5‑7B‑Instruct & 73.6 & 71.9 & 85.3 & 76.1 & 74.2 & 74.2 & 81.0 & 75.7 & 72.1 & 67.9 & 80.4 & 68.4 \\
Mistral‑8B‑Instruct‑2040& 74.3 & 72.9 & 85.2 & 75.8 & 73.5 & 73.5 & 80.8 & 75.7 & 71.1 & 66.7 & 79.7 & 67.3 \\
\bottomrule
\end{tabular}
\vskip -2ex
\label{tab-llm-baselines}
\end{table}

\paragraph{LLM Variants.}

 To further examine model variation, we conducted experiments with Qwen2.5‑7B‑Instruct and Mistral‑8B‑Instruct‑2040. The results using SaVe‑TAG
 are shown in Table \ref{tab-llm-baselines}.



\section{Implementation Details} \label{appen_sec:details}

\textbf{Code Implementation. }
\noindent Our model uses PyG \cite{fey2019fastgraphrepresentationlearning} and Sentence Transformer \cite{reimers2019sentence}. 
Experiments were conducted on Ubuntu 22.04.4 LTS with 128GB RAM and NVIDIA GeForce RTX 3090.

\noindent \textbf{Hyperparameters. }
\noindent Details of our setups for reproducibility can be found at \href{https://github.com/LWang-Laura/SaVe-TAG}{https://github.com/LWang-Laura/SaVe-TAG}.

\noindent \textbf{Datasets. }
\noindent We obtained Cora and PubMed from~\cite{chen2024exploring}; Citeseer, Photo, Children, and Computer datasets from~\cite{chen2024textspacegraphfoundationmodels} with descriptions in~\cite{yan2023comprehensive}. 

\vspace{6ex}

\noindent \textbf{Baselines Details. }

\noindent \textit{LLM-based Data Augmentation Baselines.}\\

\vspace{-2.5ex}
\begin{itemize}[leftmargin=1em]
    \item \textit{Zero-shot}~\cite{li-etal-2023-synthetic,yu2024leveraginglargelanguagemodels}: Utilizing a pre-trained LLM to generate textual attributes for a given target label without providing 
    examples.
    \item \textit{Few-shot}~\cite{li-etal-2023-synthetic,yu2024leveraginglargelanguagemodels}: Prompting an LLM to generate text 
    for a given target label, providing 
    examples for in-context learning.
\end{itemize}

\noindent \textit{Long-tailed Graph Learning Baselines.}\\

\vspace{-3ex}
\begin{itemize}[leftmargin=1em]
\item \textit{Oversampling}~\cite{chawla2003c4}: Repeatedly duplicates minority/tail nodes. 
\item \textit{SMOTE}~\cite{chawla2002smote}: Synthetic nodes interpolate 
tail nodes with neighbors and assigning edges by copying their neighbors' edges.
\item \textit{Embed-SMOTE}~\cite{ando2017deep}: Applying SMOTE on hidden layer embeddings rather than the input features.
\item \textit{MixupForGraph}~\cite{wang2021mixup}: 
Synthesizing data for minority classes by interpolating random pairs of hidden representation and labels.
\item \textit{GraphSMOTE}~\cite{zhao2021graphsmote}: Interpolating minority nodes with nearest neighbors and generating new edges via a co-trained link predictor; two variants $_T$ and $_O$, for discrete or continuous edges
\item \textit{LTE4G}~\cite{10.1145/3511808.3557381}: Clustering nodes by class and degree, training subset 
experts, and then distilling them for 
classification.
\item \textit{HierTail}~\cite{10.1145/3637528.3671880}: Extracting shared class information via hierarchical task grouping and balancing head–tail gradient contributions.

\end{itemize}

\noindent \textit{\method~Variants.}

The description of the three variants are listed in Section \ref{subsubsec:baselines} and a case study is provided in Figure~\ref{fig-DataAug}.

\begin{figure}[t!]
    \small
    \includegraphics[width=0.5\textwidth]{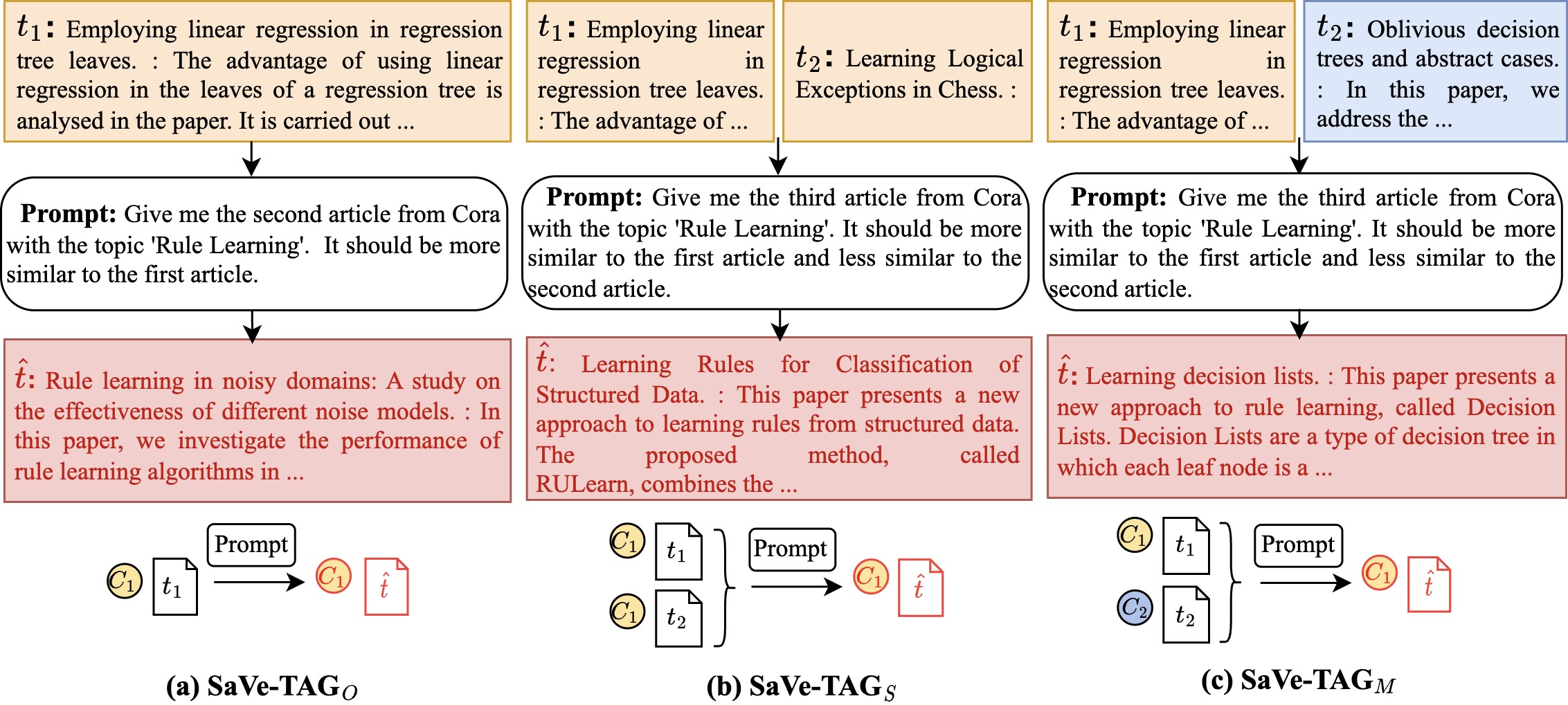} 
     \vskip -2ex
    \caption{
    A case study illustrating our three variants: 
    \\(a) \method$_O$, (b) \method$_S$, and (c) \method$_M$.
    }
    \label{fig-DataAug}
    \vskip -4ex
\end{figure}

\subsection{Prompt Design}
The prompt template for \method$_S$ and \method$_M$ share a common template in Table \ref{tab:st_mx_prompt}; \method$_O$ is shown in Table \ref{tab:prompt-o}. Templates are dataset-specific, with parameters detailed in Table \ref{tab-para}.
\vskip -2ex

\begin{table}[h!]
  \caption{Prompt template for  \method$_S$ and  \method$_M$ (\( \textcolor{blue}{\mathcal{C}_1}=\textcolor{blue}{\mathcal{C}_2}\) for  \method$_O$).}
   \vskip -2ex
  \centering
    \footnotesize
  \setlength{\fboxsep}{0pt}
  \renewcommand{\arraystretch}{0.6}   
  \begin{tabular}{@{}p{0.15\linewidth} p{0.80\linewidth}@{}}
    \toprule
    \textbf{System:}
      & You are a helpful AI assistant for generating 
        \colorbox{hlcolor}{\texttt{\{Task\}}} 
        from 
        \colorbox{hlcolor}{\texttt{\{Dataset\}}}, 
        where each 
        \colorbox{hlcolor}{\texttt{\{Text\}}} 
        follows the format 
        \texttt{<START>\colorbox{hlcolor}{\{Format\}}<END>}. \\

    \midrule
    \textbf{User:}
      & Give me the first 
        \colorbox{hlcolor}{\texttt{\{Text\}}} 
        from 
        \colorbox{hlcolor}{\texttt{\{Dataset\}}} 
        with topic 
        \textcolor{blue}{[$\mathcal{C}_1$]}. \\

    \cmidrule(l{4pt}r{4pt}){1-2}
    \textbf{Assistant:}
      & \texttt{<START>}
        \;\textcolor{blue}{[$t_1$]}\;
\texttt{<END>}.\\

    \midrule
    \textbf{User:}
      & Give me the second 
        \colorbox{hlcolor}{\texttt{\{Text\}}} 
        from 
        \colorbox{hlcolor}{\texttt{\{Dataset\}}} 
        with topic 
        \textcolor{blue}{[$\mathcal{C}_2$]}. \\

    \cmidrule(l{4pt}r{4pt}){1-2}
    \textbf{Assistant:}
      & \texttt{<START>}
        \;\textcolor{blue}{[$t_2$]}\;
\texttt{<END>}.\\

    \midrule
    \textbf{User:}
      & Give me the third 
        \colorbox{hlcolor}{\texttt{\{Text\}}} 
        from 
        \colorbox{hlcolor}{\texttt{\{Dataset\}}} 
        with topic 
        \textcolor{blue}{[$\mathcal{C}_2$]}. 
        It should be more similar to the first 
        \colorbox{hlcolor}{\texttt{\{Text\}}} 
        and less similar to the second 
        \colorbox{hlcolor}{\texttt{\{Text\}}}. \\

    \cmidrule(l{4pt}r{4pt}){1-2}
    \textbf{Assistant:}
      & \dots \\

    \bottomrule
  \end{tabular}
    \vskip -2ex

  \label{tab:st_mx_prompt}
\end{table}

\begin{table}[h!]
 \caption{Prompt template for \method$_O$ across all datasets.
 }
   \vskip -2.5ex
  \centering
  \footnotesize
  \setlength{\fboxsep}{0pt}
  \renewcommand{\arraystretch}{0.1} 
  \setlength{\tabcolsep}{3pt}                  
  \begin{tabular}{@{}p{0.15\linewidth} p{0.80\linewidth}@{}}
    \toprule
    \textbf{System:}
      & You are a helpful AI assistant for generating 
        \colorbox{hlcolor}{\texttt{\{Task\}}} 
        from 
        \colorbox{hlcolor}{\texttt{\{Dataset\}}}, 
        where each 
        \colorbox{hlcolor}{\texttt{\{Text\}}} 
        follows the format 
        \texttt{<START>\colorbox{hlcolor}{\{Format\}}<END>}. \\

    \midrule
    \textbf{User:}
      & Give me the first 
        \colorbox{hlcolor}{\texttt{\{Text\}}} 
        from 
        \colorbox{hlcolor}{\texttt{\{Dataset\}}} 
        with topic 
        \textcolor{blue}{[$\mathcal{C}_1$]}. \\

    \cmidrule(l{4pt}r{4pt}){1-2}
    \textbf{Assistant:}
      & \texttt{<START>}
        \;\textcolor{blue}{[$t_1$]}\;
\texttt{<END>}.\\

    \midrule
    \textbf{User:}
      & Give me the second 
        \colorbox{hlcolor}{\texttt{\{Text\}}} 
        from 
        \colorbox{hlcolor}{\texttt{\{Dataset\}}} 
        with topic 
        \textcolor{blue}{[$\mathcal{C}_1$]}. 
        It should be more similar to the first 
        \colorbox{hlcolor}{\texttt{\{Text\}}}. \\

    \cmidrule(l{4pt}r{4pt}){1-2}
    \textbf{Assistant:}
      & \dots \\

    \bottomrule
  \end{tabular}

  \label{tab:prompt-o}
\end{table}

\begin{table}[t!]
\vskip -1.5ex
\caption{Prompt template input parameters for each dataset.}
\vskip -2.5ex
 \begin{adjustbox}{width=0.47\textwidth}
\footnotesize
\centering
 \renewcommand{\arraystretch}{1.3} 
   \setlength{\fboxsep}{0pt}
\setlength\tabcolsep{1pt}
\setlength{\extrarowheight}{-4.5pt}
\begin{tabular}{llllll}
\toprule
    \colorbox{hlcolor}{\texttt{\{Dataset\}}} &  \colorbox{hlcolor}{\texttt{\{Task\}}}&  \colorbox{hlcolor}{\texttt{\{Text\}}}  & \colorbox{hlcolor}{\texttt{\{Format\}}} \\
\midrule
\textbf{Cora} & `new academic articles'& `article'& `[New Title] : [New Abstract]'\textbackslash n'\\

\textbf{Pubmed} & `new academic articles'& `article'& 
\begin{tabular}[t]{@{}l@{}} 
`Title: [New Title] \\
Abstract: [New Abstract]'
\end{tabular}\\

\textbf{Citeseer} & `new academic articles'& `article'& `[New Title] : [New Abstract]'\textbackslash n'\\

\textbf{Photo} & `reviews of products from Amazon'& `review'& `Review: [New Review]'\\
 
\textbf{Computer} & `reviews of products from Amazon'& `review'& `Review: [New Review]'\\

\textbf{Children}& `new book descriptions'& 
 \begin{tabular}[t]{@{}l@{}} 
 `book \\
 description'
 \end{tabular}
& \begin{tabular}[t]{@{}l@{}} 
`Title: [New Title] \\ 
Book Description: [New Description]' 
\end{tabular}\\
 \bottomrule
\end{tabular}
\end{adjustbox}

\label{tab-para}
\end{table}



\begin{table}[t!]
\vskip -1.5ex
 \caption{\textbf{Dataset Details.} 
 }
 \label{tab-dataset}
 \vspace{-2.5ex}
 \centering
  \renewcommand{\arraystretch}{1.3} 
   \setlength{\fboxsep}{0pt}
\setlength\tabcolsep{1pt}
\setlength{\extrarowheight}{-4.5pt}
 \begin{adjustbox}{width=0.4\textwidth}
\setlength{\tabcolsep}{1pt}
 \footnotesize
\begin{tabular}{lcccccccc}
\toprule
\textbf{Name}                  & \textbf{\#Nodes     }          & \textbf{\#Edges    }           & \textbf{\#Class  } & \textbf{\#Tail  } &\textbf{Len}&  \textbf{\#Batch } &\textbf{BZ}&\textbf{Time}\\ \midrule
Cora                  & 2,708& 10,858& 7 &                    5&891&  20 &4&3.06s\\
PubMed               & 19,717& 88,670& 3&                     2&1649&  8 &4&4.51s\\
Citeseer& 3186& 8450& 6& 4& 1022&  16&4&3.50s\\
Photo                 & 48,362& 500,939& 12&8&804&  64 &2&4.60s\\
Computer              & 87,229& 721,081& 10&6&499&  24 &2&4.09s\\
Children              & 76,875& 1,554,578& 24&15&1255&  240 &1&7.21s\\
\bottomrule
\end{tabular}
 \end{adjustbox}

\begin{tablenotes}
    \centering 
      \footnotesize
      \item Len: the avg. text length of node attributes; \# Batch: the total number of batches; BZ (Batch Size): the number of prompts per batch; 
      Time: the avg. time per prompt (generated entry) during generation.
 
\end{tablenotes}

 \vskip -1ex
\end{table}
\begin{table}[h]
\vskip -1ex
\centering
      \footnotesize
\caption{Runtime and memory usage across methods.}
\vskip -2.5ex
\setlength{\tabcolsep}{2pt}
\renewcommand{\arraystretch}{0.5}

\begin{tabular}{lccc ccc ccc}
\toprule
& \multicolumn{3}{c}{\textbf{Cora}} 
& \multicolumn{3}{c}{\textbf{Citeseer}} 
& \multicolumn{3}{c}{\textbf{PubMed}} \\
\cmidrule(lr){2-4}\cmidrule(lr){5-7}\cmidrule(lr){8-10}
\textbf{Method} 
& Time & RAM & GPU
& Time & RAM & GPU
& Time & RAM & GPU \\
\midrule
LTE4G     & 120.4 & 1.9 & 630 & 119.4 & 1.6 & 838 & OOM  & OOM & OOM \\
HierTail  & 67.2  & 1.5 & 630 & 49.8  & 1.6 & 838 & OOM  & OOM & OOM \\
SaVe-TAG  & 10.3  & 1.7 & 738 & 11.4  & 1.7 & 738 & 24.9 & 1.8 & 954 \\
\bottomrule
\end{tabular}

\begin{tablenotes}
    \centering 
      \footnotesize
      \item Note that Time is in seconds, RAM in GB, and GPU  in MB. 
 
\end{tablenotes}

\vskip -2ex
\label{tab:runtime}
\end{table}

\section{Efficiency and Scalability Analysis} \label{appen_sce:analysis}

\subsection{Time Costs}
First, we report the statistics regarding the graph sizes 
in 
Table \ref{tab-dataset}.

For scalability analysis and examine efficiency, we measure post-augmentation runtime of our method, and report peak RAM, peak GPU usage, and wall-clock time for each run. We compare SaVe-TAG with SOTA baselines LTE4G and HierTail. Notably, both LTE4G and HierTail encountered OOM issues on PubMed (We manually optimized their code to obtain the results reported in 
Table \ref{tab:runtime}). 

To facilitate text generation, we generate samples in batches. While larger batches reduce overall runtime, memory constraints limit us to small batch sizes. We then compute the average LLM inference time per generated sample. The data regarding batch size and inference time per prompt are also reported in Table~\ref{tab-dataset}.

\subsection{Adaptability for Different LLM Variants}

Figure~\ref{fig-llm} shows the performance and efficiency of \method$_S$ using LLMs with different sizes, including Llama-3-8B-Instruct, Llama-3.2-1B-Instruct, and Llama-3.2-3B-Instruct. As expected, smaller models require less inference time for text generation. Notably, the F1 score remains relatively stable across model sizes, highlighting the adaptability of our method to lightweight, open-source LLMs—a promising direction for reducing generation costs.

\begin{figure}
    \centering
    \includegraphics[width=0.9\columnwidth]{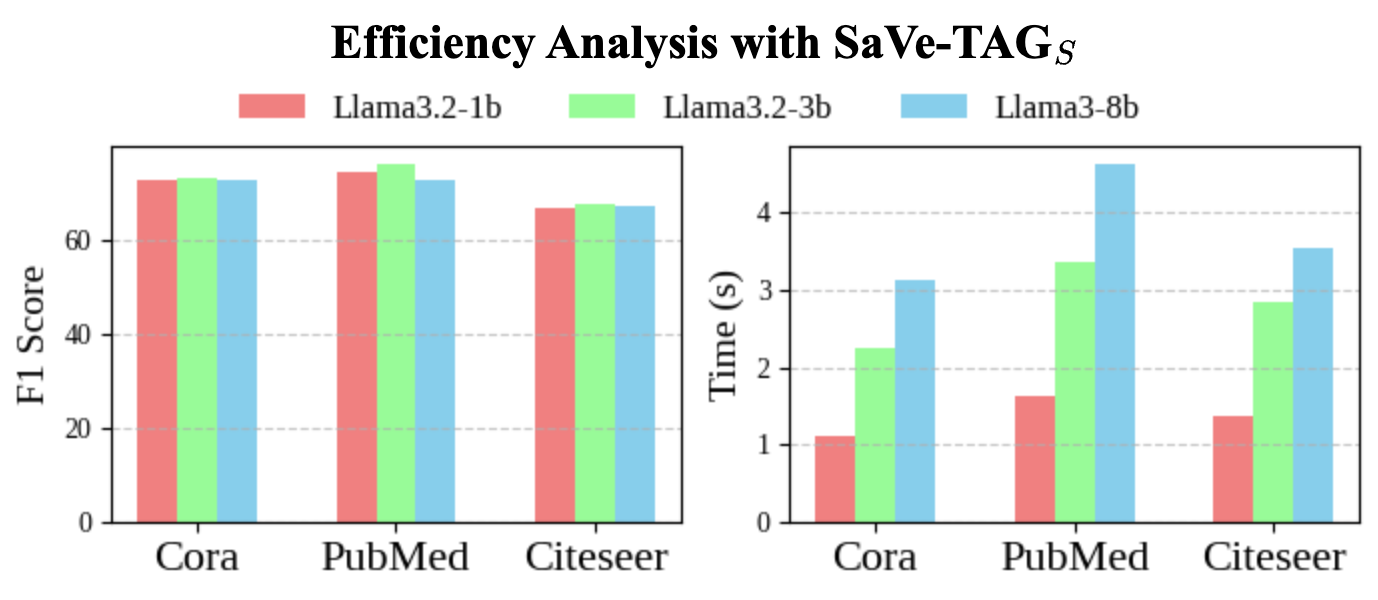}
    \vskip -2ex
    \caption{The F1 score and generation time per entry of LLM variants across different datasets implementing \method$_S$.}
    \label{fig-llm}
    \vskip -1ex
\end{figure}






\section{Detailed Proofs} \label{app_sec:add_theory}

\subsection{Proof of Theorem~\ref{thm:manifold-preservation}
            }
\label{appen:proof_manifold_preservation}
\textbf{Manifold-Preserving Class-Consistent Generation. }


\noindent \textit{Assumptions for this theorem.}
\begin{enumerate}[label=\textbf{A\arabic*}., wide=0pt, leftmargin=2.6em]
\item\label{asm:real-mass} \emph{True on-manifold mass.}\;
      $p_c(A) \ge 1-\delta$ for some $\delta\in[0,1)$.
\item\label{asm:tv-bound} \emph{LLM approximation in TV.}\;
      $D_{\mathrm{TV}}(p_c, q_c) \le \epsilon$ for some $\epsilon\in[0,1)$.
\item\label{asm:measurable} \emph{Measurability.}\;
      $A$ is measurable so that $p_c(A)$ and $q_c(A)$ are well defined.\footnote{%
      This holds automatically when $\phi$ is measurable and $\mathcal M_c$ is Borel.}
\end{enumerate}

\paragraph{Proof.}
By Assumption~\ref{asm:tv-bound} and the definition of total variation,
$
  q_c(A)
  \;\ge\;
  p_c(A) - D_{\mathrm{TV}}(p_c,q_c).
$
Applying Assumption~\ref{asm:real-mass} yields
$
  q_c(A)
  \;\ge\;
  (1-\delta) - \epsilon
  \;=\;
  1 - (\delta+\epsilon).
$
By the definition of $A$ and $q_c$,
$q_c(A)=\Pr_{\hat t\sim \LLM(\cdot\mid c)}\!\big[\phi(\hat t)\in\mathcal M_c\big]$,

which proves
$
  \Pr\big[\phi(\hat t)\in\mathcal M_c\big] \;\ge\; 1-(\delta+\epsilon).
\quad\qedhere
$

\subsection{Proof of Theorem~\ref{thm:margin}}
\label{app_sec:f1}
\textbf{Margin Lower Bound. }


\noindent \textit{Assumptions for this theorem.}
\begin{enumerate}[label=\textbf{A\arabic*}., wide=0pt, leftmargin=2.4em]
\item \emph{Margin monotonicity.}\label{asm:monotone}
      Retraining on $\tilde{\mathcal D}$ does not decrease any
      margin that was already $\ge\gamma_0$.
      Thus, every point in
      $\mathcal D\cup\widehat{\mathcal D}_{\mathrm{in}}$
      keeps margin $\ge\gamma_0$ under the new model
      $f_\theta^{\mathrm{aug}}$.
\item \emph{Boundary–sample slack.}\label{asm:delta}
      For every boundary sample $\hat x$ we have the (possibly loose) bound
      \(\lvert\gamma(\hat x)\rvert\le\delta\) with a fixed
      constant $\delta>0$.
\item \emph{Conservative ordering.}\label{asm:ordering}
      We choose $\delta\!\ge\!\gamma_0$.
      This turns \ref{asm:delta} into a \emph{trivial worst-case cap}
      (it can never be violated) and lets the algebra below
      produce a bona-fide lower bound—that is, the right‐hand side
      never exceeds the true minimum margin.

\end{enumerate}

\paragraph{Two exhaustive cases.}
Let
\(m_{\mathrm{bd}}:=|\widehat{\mathcal D}_{\mathrm{bd}}|\) and
\(m_{\mathrm{in}}:=|\widehat{\mathcal D}_{\mathrm{in}}|\).

\medskip
\noindent
\emph{Case 1: \(\boldsymbol{m_{\mathrm{bd}}=0}\) (BCR $=0$).}
The augmented set contains only
$\mathcal D\cup\widehat{\mathcal D}_{\mathrm{in}}$.
By Assumption~\ref{asm:monotone} every margin is still
$\ge\gamma_0$, hence
\(
\gamma_{\min}(\tilde{\mathcal D})\ge\gamma_0.
\)
The bound in the statement reduces to
\(
\gamma_0-\delta(1-0)=\gamma_0,
\)
so the inequality is tight in this corner case.

\medskip
\noindent
\emph{Case 2: \(\boldsymbol{m_{\mathrm{bd}}>0}\) (BCR $>0$).}
Now at least one boundary sample is present.
By \ref{asm:delta}  
its margin is $\ge-\delta$ and, because $\delta\ge\gamma_0$
(\ref{asm:ordering}), also
$\ge\gamma_0-\delta$.
Combining with Assumption~\ref{asm:monotone} for all other points gives
\[
\gamma_{\min}(\tilde{\mathcal D})
\;\ge\;
\min\bigl\{\gamma_0,\;\gamma_0-\delta\bigr\}
\;=\;
\gamma_0-\delta.
\]
Because BCR\(\in(0,1]\),
\(
\gamma_0-\delta(1-\mathrm{BCR})
\;\le\;
\gamma_0-\delta,
\)
and the claimed inequality again holds.

\paragraph{Unifying the two cases.}
Observe that in \textbf{Case 1}
$(1-\mathrm{BCR})=1$ while in \textbf{Case 2}
$(1-\mathrm{BCR})<1$.
Hence the single formula
\[
\boxed{\;
\gamma_{\min}\!\bigl(\tilde{\mathcal D}\bigr)
\;\;\ge\;\;
\gamma_{0}-\delta\bigl(1-\mathrm{BCR}\bigr)
\;}
\]
is simultaneously valid for both scenarios, completing the proof.

\subsection{Proof of Theorem~\ref{thm:manifold-vrm}
            }
\label{appen:proof_vrm}

\textbf{On-Manifold Vicinal Risk Theorem. }


\noindent \textit{Assumptions for this theorem.}
\begin{enumerate}[label=\textbf{A\arabic*}., wide=0pt, leftmargin=2.6em]
\item\label{asm:manifold} \emph{Manifold-preservation.}\;
      For every $i$ and any neighbor
      $\hat t\sim\mathcal V_{\LLM}(t_i)$,
      \[
      \Pr\!\bigl[\phi(\hat t)\!\notin\!\mathcal M_{c_i}\bigr]
      \;\le\;\delta,
      \]
      where $\delta\in(0,1)$ is the constant proved in
      Theorem~\ref{thm:manifold-preservation}.
\item\label{asm:uniformdelta} \emph{Uniform bound.}\;
      The same $\delta$ works for \emph{all} training indices $i$
      (if they differ, use the worst-case
      $\delta=\max_i\delta_i$).
\item\label{asm:lossbound} \emph{Finite loss.}\;
      $\mathbf{L}(f_\theta(\phi(\hat t)),c_i)$ is integrable so that all
      expectations are well defined.\footnote{%
      A bounded loss (e.g.\ cross-entropy with label smoothing) or a
      sub-Gaussian surrogate satisfies this automatically.}
\end{enumerate}

\paragraph{Vicinal risk decomposition.}
The VRM objective is
\[
R_{vrm}(f_\theta)
=\frac1n\sum_{i=1}^{n}
  \mathbb{E}_{\hat t\sim\mathcal V_{\LLM}(t_i)}
    \bigl[\mathbf{L}(f_\theta(\phi(\hat t)),c_i)\bigr].
\]

Fix $i$ temporarily.
Apply \textbf{\ref{asm:manifold}} and law of total expectation \cite{grimmett2020probability}:
\[
\mathbb{E}_{\hat t}\bigl[\mathbf{L}(\cdot)\bigr]
=(1-\delta)\,
   \mathbb{E}\!\bigl[\mathbf{L}(\cdot)\mid\phi(\hat t)\!\in\!\mathcal M_{c_i}\bigr]
+\;\delta\,
   \mathbb{E}\!\bigl[\mathbf{L}(\cdot)\mid\phi(\hat t)\!\notin\!\mathcal M_{c_i}\bigr].
\]

Define
\[
R_{\text{on}}
:=\frac1n\sum_{i=1}^{n}
  \mathbb{E}\!\bigl[\mathbf{L}(\cdot)\mid\phi(\hat t)\in\mathcal M_{c_i}\bigr],
R_{\text{off}}
:=\frac1n\sum_{i=1}^{n}
  \mathbb{E}\!\bigl[\mathbf{L}(\cdot)\mid\phi(\hat t)\notin\mathcal M_{c_i}\bigr].
\]

Summing the displayed equality over $i$ 
and using
\textbf{\ref{asm:uniformdelta}} yields
\[
\boxed{\;
R_{vrm}(f_\theta)
=(1-\delta)\,R_{\text{on}}+\delta\,R_{\text{off}}
\;}
\]
with $0\le\delta\ll1$ coming from
Theorem~\ref{thm:manifold-preservation}.
Because $(1-\delta)$ is the probability of drawing an
\emph{on-manifold} neighbor, the VRM risk is evaluated on-manifold
with probability at least $1-\delta$, completing the proof.\qedhere

\subsection{Proof of Theorem~\ref{thm:margin-vrm}	}
\label{appen:proof_margin_vrm}

\textbf{Boundary-Coverage $\Rightarrow$ Vicinal-Risk Reduction. }


\noindent \textit{Assumptions used in the proof.}
\begin{enumerate}[label=\textbf{A\arabic*}., wide=0pt, leftmargin=2.6em]
\item\label{asm:lipschitz}
      \emph{Lipschitz loss.} 
      $\mathcal J(z_1,y)-\mathcal J(z_2,y)\le L\|z_1-z_2\|_2$ ~\cite{shalev2014understanding}.
\item\label{asm:margin_inc}
      \emph{Non-decreasing margins.} 
      VRM retraining never \emph{reduces} the margin of a synthetic
      point: $\gamma_{\,\mathrm{aug}}(\hat x)\ge\gamma_{\,\mathrm{orig}}(\hat x)$.
\item\label{asm:bd_slack}
      \emph{Boundary slack.} 
      $|\gamma(\hat x)|\le\delta$ for all
      $\hat x\in\widehat{\mathcal D}_{\mathrm{bd}}$.
\item\label{asm:int_margin}
      \emph{Interior margin.} 
      $\gamma(\hat x)\ge\gamma_0$ for all
      $\hat x\in\widehat{\mathcal D}_{\mathrm{in}}$.
\end{enumerate}

(Last three assumptions are justified in Appendix \ref{app_sec:f1}.)

\paragraph{Pointwise loss drop.}
For any synthetic $(\hat x,\hat y)$, combine
\textbf{\ref{asm:lipschitz}} + \textbf{\ref{asm:margin_inc}}
to obtain
\[
\mathcal J\bigl(f_\theta^{\mathrm{aug}}(\hat x),\hat y\bigr)
-\mathcal J\bigl(f_\theta^{\mathrm{orig}}(\hat x),\hat y\bigr)
\;\le\;
-L\,\gamma(\hat x).
\tag{S1} \label{eq:S1}
\]

\paragraph{Separate boundary vs. interior sums.}
Sum \eqref{eq:S1} over the two synthetic subsets and divide by $m$:
\begin{align*}
R_{vrm}\!\bigl(f_\theta^{\mathrm{aug}}\bigr)
- R_{vrm}\!\bigl(f_\theta^{\mathrm{orig}}\bigr)
&\le
-\frac{L}{m}
 \sum_{\hat x\in\widehat{\mathcal D}_{\mathrm{bd}}}\!\!\gamma(\hat x)
-\frac{L}{m}
 \sum_{\hat x\in\widehat{\mathcal D}_{\mathrm{in}}}\!\!\gamma(\hat x) \\[4pt]
&\overset{\textbf{\ref{asm:bd_slack}},\textbf{\ref{asm:int_margin}}}{\le}
-L\,\delta\,
   \frac{|\widehat{\mathcal D}_{\mathrm{bd}}|}{m}
\;-\;
L\,\gamma_0\,
   \frac{|\widehat{\mathcal D}_{\mathrm{in}}|}{m}.
\end{align*}

Recognize the fractions as \(\mathrm{BCR}\) and \(1-\mathrm{BCR}\):
\[
R_{vrm}\!\bigl(f_\theta^{\mathrm{aug}}\bigr)
- R_{vrm}\!\bigl(f_\theta^{\mathrm{orig}}\bigr)
\;\le\;
-L\,\gamma_0\,\mathrm{BCR}
\;+\;
L\,\delta\,(1-\mathrm{BCR}).
\tag{S2} \label{eq:S2}
\]

\paragraph{Express the bound in compact form.}
Choose 
\(
\eta=\delta/\gamma_0
\;(0<\eta<1)
\)
and rewrite \eqref{eq:S2} as
\[
R_{vrm}\!\bigl(f_\theta^{\mathrm{aug}}\bigr)
- R_{vrm}\!\bigl(f_\theta^{\mathrm{orig}}\bigr)
\;\le\;
-L\,\gamma_0\,
   \bigl[\mathrm{BCR}-\eta\bigr]
\;+\;
\mathcal O(\delta),
\]
where the $\mathcal O(\delta)$ term hides constants independent of
$\mathrm{BCR}$ when $\delta$ is small.
Hence whenever
\(\mathrm{BCR}>\eta=\delta/\gamma_0\)
(the typical case in our experiments),
the vicinal risk strictly \emph{decreases} after VRM training.
\qedhere

\subsection{Proof of Theorem~\ref{thm:pull}}
\label{appen:proof_pull}

\textbf{Pulling Well-Aligned Nodes. }


\paragraph{Layer-wise update.}
For node $\hat v$ the generic message-passing rule (Equation \ref{eq:agg}) reads
\[
h^{(\ell+1)}_{\hat v}
=\alpha\,\W h^{(\ell)}_{\hat v}
+(1-\alpha)
  \sum_{u\in\mathcal N'(\hat v)}\beta_{\hat v u}\,\W h^{(\ell)}_u.
\tag{U}
\]

\noindent \textit{Assumptions used in the proof.}
\begin{enumerate}[label=\textbf{A\arabic*}., wide=0pt,leftmargin=2.6em]
\item\label{asm:contract}
      \emph{Non-expansive layer:}\;
      spectral norm $\|\W\|_2\le L<1$.
\item\label{asm:neigh}
      \emph{Well-aligned neighbors:}\;
      $\operatorname{dist}\!\bigl(h^{(\ell)}_u,\mathcal M_c\bigr)
      \le\varepsilon$ for every $u\in\mathcal N'(\hat v)$.
\item\label{asm:weights}
      \emph{Convexity of weights:}\;
      $\alpha\in(0,1)$ and $\{\beta_{\hat v u}\}$ form a convex
      combination.
\end{enumerate}

(All three conditions are satisfied by common GNNs 
and our confidence-based edge assignment.)

\paragraph{Contractive bounds after applying $\W$.}
By \textbf{\ref{asm:contract}}, linear mapping $\W$ shrinks distances
by at most factor $L$:
\[
\operatorname{dist}\bigl(\W h^{(\ell)}_{\hat v},\W\mathcal M_c\bigr)
      \le L\,d_\ell,
\qquad
\operatorname{dist}\bigl(\W h^{(\ell)}_{u},\W\mathcal M_c\bigr)
      \le L\,\varepsilon
      \; (\text{by \textbf{\ref{asm:neigh}}}).
\tag{B}
\]

\paragraph{Distance after aggregation.}
Because the right-hand side of (U) is a convex combination
(\textbf{\ref{asm:weights}}),
the triangle inequality and (B) give
\[
d_{\ell+1}
:=\operatorname{dist}\bigl(h^{(\ell+1)}_{\hat v},\mathcal M_c\bigr)
\;\le\;
\alpha L d_\ell \;+\;(1-\alpha)L\varepsilon.
\tag{C}
\]

\paragraph{Geometric contraction (when $d_\ell\!\ge\!\varepsilon$).}
If the current distance satisfies $d_\ell\ge\varepsilon$, then
$(1-\alpha)L\varepsilon \le(1-\alpha)L d_\ell$,
so (C) becomes
\[
d_{\ell+1}
\le
\underbrace{\bigl[\alpha L + (1-\alpha)L\bigr]}_{=\;\alpha L}\,d_\ell
=:\gamma\,d_\ell,
\]
with
$
\gamma:=\alpha L$ and $0<L<1,\;0<\alpha<1
\;\Longrightarrow\; \gamma\in(0,1).
$

\paragraph{Exponential convergence.}
Iterate the inequality whenever $d_k\ge\varepsilon$:
\[
d_\ell\;\le\;\gamma^{\ell-\ell_0}\,d_{\ell_0}
\quad(\ell\!\ge\!\ell_0),
\]
demonstrating exponential decay toward $\mathcal M_c$.
Once the distance falls below $\varepsilon$, neighbors and self-term
are already within the same $\varepsilon$-tube, so all subsequent
distances stay bounded by $L\varepsilon$.
\qedhere


\end{document}